\documentclass[twoside]{article}

\usepackage[accepted]{aistats2021}
\usepackage[round]{natbib}

\usepackage[dvipsnames]{xcolor}

\usepackage{hyperref} 
\hypersetup{
	colorlinks=true,
	citecolor=Blue,
    linkcolor=Blue,
	urlcolor=Blue,  
}

\usepackage{amsfonts}       
\usepackage{amsmath}
\usepackage{amssymb}
\usepackage{amsthm}
\usepackage{graphicx}
\usepackage{latexsym}
\usepackage{mathabx}
\usepackage{method}
\usepackage[labelformat=simple]{subcaption}

\usepackage[normalem]{ulem}
\usepackage{etoolbox}
\usepackage{wrapfig}
\usepackage{url}
\usepackage{pifont}

\usepackage{caption}

\usepackage{algorithm}
\usepackage[noend]{algorithmic}

\def\undertilde#1{\mathord{\vtop{\ialign{##\crcr
				$\hfil\displaystyle{#1}\hfil$\crcr\noalign{\kern1.5pt\nointerlineskip}
				$\hfil\widetilde{}\hfil$\crcr\noalign{\kern1.5pt}}}}}

\newcommand{\argmin}{\mathop{\mathrm{argmin}}}

\newcommand{\vs}{\boldsymbol{\mathrm{s}}}
\newcommand{\va}{\boldsymbol{\mathrm{a}}}

\newcommand{\vx}{\boldsymbol{\mathrm{x}}}

\newcommand{\vtau}{\boldsymbol{\mathrm{\tau}}}

\newcommand{\da}{\mathrm{d}\va}

\theoremstyle{plain}
\newtheorem{assumption}{Assumption}
\newtheorem{lemma}{Lemma}
\newtheorem{theorem}{Theorem}

\newenvironment{sproof}{%
  \proof}{\endproof}

\newlength\myindent
\begin{document}

%

%

\twocolumn[

\aistatstitle{Robust Imitation Learning from Noisy Demonstrations}

\aistatsauthor{ Voot Tangkaratt \And Nontawat Charoenphakdee \And Masashi Sugiyama }

\aistatsaddress{ RIKEN\\ \texttt{voot.tangkaratt@riken.jp}
\And University of Tokyo \& RIKEN\\ \texttt{nontawat@ms.k.u-tokyo.ac.jp}
\And RIKEN \& University of Tokyo\\ \texttt{sugi@k.u-tokyo.ac.jp} }]


\begin{abstract}
	Robust learning from noisy demonstrations is a practical but highly challenging problem in imitation learning.
	In this paper, we first theoretically show that robust imitation learning can be achieved by optimizing a classification risk with a symmetric loss.
	Based on this theoretical finding, we then propose a new imitation learning method that optimizes the classification risk by effectively combining pseudo-labeling with co-training.
	Unlike existing methods, our method does not require additional labels or strict assumptions about noise distributions.
	Experimental results on continuous-control benchmarks show that our method is more robust compared to state-of-the-art methods. 
\end{abstract}

\section{INTRODUCTION}


The goal of sequential decision making is to learn a good policy that makes good decisions~\citep{Puterman1994}. 
Imitation learning (IL) is an approach that learns a policy from demonstrations (i.e., sequences of demonstrators' decisions)~\citep{Schaal1999}. 
Researchers have shown that a good policy can be learned efficiently from high-quality demonstrations collected from experts~\citep{NgR00,HoE16}.
However, demonstrations in the real-world often have lower quality due to noise or insufficient expertise of demonstrators, especially when humans are involved in the data collection process~\citep{MandlekarZGBSTG18}. 
This is problematic because low-quality demonstrations can reduce the efficiency of IL both in theory and practice~\citep{TangkarattEtAl2020}. 
In this paper, we theoretically and experimentally show that IL can perform well even in the presence of noises.

In the literature, methods for IL from noisy demonstrations have been proposed, but they still have limitations as they require additional labels from experts or a strict assumption about noise~\citep{BrownGNN19,BrownEtAl2020,Wu2019,TangkarattEtAl2020}. 
Specifically, methods of~\cite{BrownGNN19,BrownEtAl2020} require noisy demonstrations to be ranked according to their relative performance. 
Meanwhile, methods of~\cite{Wu2019} require some of noisy demonstrations to be labeled with a score determining the probability that demonstrations are collected from experts.
On the other hand, the method of~\cite{TangkarattEtAl2020} does not require these labels, but instead it assumes that noisy demonstrations are generated by Gaussian noise distributions.
To sum up, these methods require either additional labels from experts or a strict assumption about noise distributions.
Due to this, the practicality of these methods is still limited, and IL from noisy demonstrations is still highly challenging. 

To overcome the above limitation, we propose a new method for IL from noisy demonstrations called \emph{Robust IL with Co-pseudo-labeling} (RIL-Co).
Briefly, we built upon the recent theoretical results of robust classification~\citep{charoenphakdee19a}, and prove that robust IL can be achieved by optimizing a classification risk with a \emph{symmetric loss}. 
However, optimizing the proposed risk is not trivial because it contains a data density whose data samples are not observed. 
We show that pseudo-labeling~\citep{Chapelle2010} can be utilized to estimate the data density. 
However, naive pseudo-labeling may suffer from over-fitting and is not suitable in practice~\citep{Kingma2014}.
To remedy this issue, we propose \emph{co-pseudo-labeling}, which effectively combines pseudo-labeling with co-training~\citep{blum98}.
Compare to prior work, RIL-Co does not require additional labels or assumptions about noise distributions. 
In addition, RIL-Co does not require an additional hyper-parameter tuning because an appropriate hyper-parameter value can be derived from the theory. 
Experiments on continuous-control benchmarks show that RIL-Co is more robust against noisy demonstrations when compared to state-of-the-art methods.

\section{IMITATION LEARNING AND ROBUSTNESS}

In this section, we firstly give backgrounds about reinforcement learning and imitation learning.
Then, we describe the setting of imitation learning from noisy demonstrations. 
Lastly, we discuss the robustness of existing imitation learning methods. 

\subsection{Reinforcement Learning}

Reinforcement learning (RL) aims to learn an optimal policy of a Markov decision process (MDP)~\citep{Puterman1994}.
We consider a discrete-time MDP denoted by $\mathcal{M} = (\mathcal{S}, \mathcal{A}, p_\mathrm{T}(\vs^\prime|\vs,\va)), p_1(\vs_1), r(\vs,\va), \gamma)$ with state $\vs \in \mathcal{S}$, action $\va \in \mathcal{A}$, transition probability density $p_{\mathrm{T}}$, initial state probability density $p_1$, reward function $r$, and discount factor $0 < \gamma \leq 1$. 
A policy function $\pi(\va|\vs)$ determines the conditional probability density of an action in a state.
An agent acts in an MDP by observing a state, choosing an action according to a policy, transiting to a next state according to the transition probability density, and possibly receiving an immediate reward according to the reward function. An optimal policy of an MDP is a policy maximizing the expected cumulative discounted rewards.

Formally, RL seeks for an optimal policy by solving the optimization problem $\max_{\pi} \mathcal{J}(\pi)$, where $ \mathcal{J}(\pi)$ is the expected cumulative discounted rewards defined as
\begin{align}
	\mathcal{J}(\pi) 
	&= \mathbb{E}_{p_{\pi}}\left[ \textstyle{\sum}_{t=1}^T \gamma^{t-1} r(\vs_t,\va_t) \right] \notag \\
	&= \mathbb{E}_{\rho_\pi}\left[ r(\vs_t,\va_t) \right]/(1-\gamma).	\label{eq:rl_obj2}
\end{align}
Here, $p_{\pi}(\tau) = p_1(\vs_1) \Pi_{t=1}^T p_\mathrm{T}(\vs_{t+1}|\vs_t,\va_t) \pi(\va_t|\vs_t) $ is the probability density of trajectory $\vtau = (\vs_1, \va_1, \dots, \vs_{T+1})$ with length $T$, $\rho_{\pi}(\vs,\va) \!=\! (1-\gamma)\mathbb{E}_{p_{\pi}(\tau)}[\Sigma_{t=1}^T \gamma^{t-1} \delta( \vs_t-\vs,\va_t-\va )]$ is the state-action density determining the probability density of the agent with $\pi$ observing $\vs$ and executing $\va$, and $\delta( \vs_t-\vs,\va_t-\va )$ is the Dirac delta function. 
Note that the state-action density is a normalized occupancy measure and uniquely corresponds to a policy by a relation $\pi(\va|\vs) = \rho_{\pi}(\vs,\va) / \rho_{\pi}(\vs)$, where $\rho_{\pi}(\vs) = \int_{\mathcal{A}} \rho_{\pi} (\vs,\va) \da$~\citep{SyedBS08}. 
Also note that the optimal policy is not necessarily unique since there can be different policies that achieve the same expected cumulative discounted rewards. 

While RL has achieved impressive performance in recent years~\citep{SilverEtAl2017}, its major limitation is that it requires a suitable reward function which may be unavailable in practice~\citep{Schaal1999}.

\subsection{Imitation Learning}

Imitation learning (IL) is a well-known approach to learn an optimal policy when the reward function is unavailable~\citep{NgR00}.
Instead of learning from the reward function or reward values, IL methods learn an optimal policy from demonstrations that contain information about an optimal policy. 
IL methods typically assume that demonstrations (i.e., a dataset of state-action samples) are collected by using an expert policy that is similar to an optimal (or near optimal) policy, and they aim to learn the expert policy~\citep{NgR00,ZiebartEtAl2010,Sun2019}. 
More formally, the typical goal of IL is to learn an expert policy $\pi_{\mathrm{E}}$ by using a dataset of state-action samples drawn from an expert state-action density:
\begin{align}
	\{ (\vs_n, \va_n) \}_{n=1}^N \overset{\mathrm{i.i.d.}}{\sim} \rho_{\mathrm{E}}(\vs, \va),
\end{align}
where $\rho_{\mathrm{E}}$ is a state-action density of expert policy $\pi_{\mathrm{E}}$. 

The {density matching} approach was shown to be effective in learning the expert policy from expert demonstrations~\citep{SyedBS08,HoE16,ghasemipour20a}.
Briefly, this approach seeks for a policy $\pi$ that minimizes a divergence between the state-action densities of the expert and learning policies. 
Formally, this approach aims to solve the following optimization problem: 
\begin{align}
\min_{\pi} D(\rho_{\mathrm{E}}(\vs,\va) || \rho_{\pi}(\vs,\va)),
\end{align}
where $D$ is a divergence such as the Jensen-Shannon divergence\footnote{For non-symmetric divergence, an optimization problem $\min_{\pi} D(\rho_{\pi}(\vs,\va) || \rho_{\mathrm{E}}(\vs,\va) )$ can be considered as well~\citep{ghasemipour20a}.}.
In practice, the divergence, which contains unknown state-action densities, is estimated by using demonstrations and trajectories drawn from $\rho_{\mathrm{E}}$ and $\rho_{\pi}$, respectively. 
A well-known density matching method is generative adversarial IL (GAIL)~\citep{HoE16}, which minimizes an estimate of the Jensen-Shannon divergence by solving 
\begin{align}
	\min_{\pi} \max_{g} 
	&~\mathbb{E}_{\rho_{\mathrm{E}}}\left[ \log \left( \frac{1}{1 + \exp(-g(\vs,\va))} \right) \right] \notag \\
	&+ \mathbb{E}_{\rho_\pi}\left[ \log \left( \frac{1}{1 + \exp(g(\vs,\va))} \right)  \right], \label{eq:gail}
\end{align}
where $g:\mathcal{S}\times \mathcal{A} \mapsto \mathbb{R}$ is a discriminator function.

Density matching methods were shown to scale well to high-dimensional problems when combined with deep neural networks~\citep{HoE16,ghasemipour20a}.
However, an issue of this approach is that it is not robust against noisy demonstrations in general, as will be described in Section~\ref{section:gail_robustness}.

\subsection{Learning from Noisy Demonstrations} \label{section:noise_data}

In this paper, we consider a scenario of \emph{IL from noisy demonstrations}, where given demonstrations are a mixture of expert and non-expert demonstrations. 
We assume that we are given a dataset of state-action samples drawn from a noisy state-action density:
\begin{align}
	\mathcal{D} = \{ (\vs_n, \va_n) \}_{n=1}^N \overset{\mathrm{i.i.d.}}{\sim} \rho^\prime(\vs,\va), \label{eq:data_1}
\end{align}
where the noisy state-action density $\rho'$ is a mixture of the expert and non-expert state-action densities:
\begin{align}
	\rho'(\vs,\va) = \alpha \rho_\mathrm{E}(\vs,\va) + (1-\alpha) \rho_{\mathrm{N}}(\vs,\va). \label{eq:data_2}
\end{align}
Here, $0.5 < \alpha < 1$ is an unknown mixing coefficient and $\rho_{\mathrm{N}}$ is the state-action density of a non-expert policy $\pi_{\mathrm{N}}$.
The policy $\pi_{\mathrm{N}}$ is non-expert in the sense that $\mathbb{E}_{\rho_{\mathrm{N}}}[r(\vs,\va)] < \mathbb{E}_{\rho_{\mathrm{E}}}[r(\vs,\va)]$, where $r$ is an unknown reward function of the MDP. 
Our goal is to learn the expert policy using the dataset in Eq.~\eqref{eq:data_1}.

We emphasize that $0.5 < \alpha < 1$ corresponds to an assumption that the majority of demonstrations are obtained by the expert policy. This is a typical assumption when learning from noisy data, i.e., the number of good quality samples should be more than that of low quality samples~\citep{Angluin1988,Nagarajan2013}. 
For notational brevity, we denote a state-action pair by $\vx = (\vs,\va)$, where $\vx \in \mathcal{X}$ and $\mathcal{X} = \mathcal{S} \times \mathcal{A}$. 

\subsection{Robustness of Imitation Learning} \label{section:gail_robustness}

It can be verified that the density matching approach is not robust against noisy demonstrations according to the data generation assumption in Eq.~\eqref{eq:data_2}.
Specifically, given demonstrations drawn from $\rho^\prime$, the density matching approach would solve $\min_{\pi} D(\rho^\prime(\vx) || \rho_{\pi}(\vx))$. 
By assuming that the space of $\pi$ is sufficiently large, the solution of this optimization problem is 
\begin{align} 
	\pi^\star(\va|\vs) &= \pi_{\mathrm{E}}(\va|\vs) \left( \frac{\alpha \rho_{\mathrm{E}}(\vx)}{\alpha\rho_{\mathrm{E}}(\vx) + (1-\alpha)\rho_{\mathrm{N}}(\vx)} \right) \notag \\
	&\phantom{=}+ \pi_{\mathrm{N}}(\va|\vs) \left( \frac{(1-\alpha) \rho_{\mathrm{N}}(\vx)}{\alpha\rho_{\mathrm{E}}(\vx) + (1-\alpha)\rho_{\mathrm{N}}(\vx)} \right),
\end{align} 
which yields $D(\rho^\prime(\vx) || \rho_{\pi^\star}(\vx)) = 0$. 
However, this policy is not equivalent to the expert policy unless $\alpha=1$. 
Therefore, density matching is not robust against noisy demonstrations generated according to Eq.~\eqref{eq:data_2}. 

We note that the data generation assumption in Eq.~\eqref{eq:data_2} has been considered previously by~\cite{Wu2019}.
In this prior work, the authors proposed a robust method that learns the policy by solving 
\begin{align}
\min_{\pi} D(\rho^\prime(\vx) || \alpha \rho_{\pi}(\vx) + (1-\alpha)\rho_{\mathrm{N}}(\vx)).
\end{align}
The authors showed that this optimization problem yields the expert policy under the data generation assumption in Eq.~\eqref{eq:data_2}. 
However, solving this optimization problem requires $\alpha$ and $\rho_{\mathrm{N}}$ which are typically unknown.
To overcome this issue, \cite{Wu2019} proposed to estimate $\alpha$ and $\rho_{\mathrm{N}}$ by using  additional noisy demonstrations that are labeled with a score determining the probability that demonstrations are drawn from $\rho_{\mathrm{E}}$. 
However, this method is not applicable in our setting since labeled demonstrations are not available.

\section{ROBUST IMITATION LEARNING}

In this section, we propose our method for robust IL.
Briefly, in Section~\ref{section:rail}, we propose an IL objective which optimizes a classification risk with a symmetric loss and prove its robustness. 
Then, in Section~\ref{section:pseudo_label}, we propose a new IL method that utilizes co-pseudo-labeling to optimize the classification risk. 
Lastly, we discuss the choice of a hyper-parameter in Section~\ref{section:lambda} and the choice of symmetric losses in Section~\ref{section:loss}.

\subsection{Imitation Learning via Risk Optimization} 
\label{section:rail}

Classification risks are fundamental quantities in classification~\citep{hastie01}. 
We are interest in a \emph{balanced risk} for binary classification where the class prior is balanced~\citep{Brodersen2010}. 
Specifically, we propose to perform IL by solving the following risk optimization problem:
\begin{align} 
	\max_{\pi} \min_{g} \mathcal{R}(g; \rho^\prime, \rho_{\pi}^\lambda, \ell_{\mathrm{sym}}), 
	\label{eq:rail_obj}
\end{align}	
where $\mathcal{R}$ is the balanced risk defined as 
\begin{align*} 
	\mathcal{R}(g; \rho^\prime, \rho_{\pi}^\lambda, \ell) =  \frac{1}{2}\mathbb{E}_{\rho^\prime} \left[  \ell(g(\vx))\right] 
	+ \frac{1}{2}\mathbb{E}_{\rho_{\pi}^\lambda} \left[  \ell(-g(\vx))\right],
\end{align*}	
and $\rho_{\pi}^\lambda$ is a mixture density defined as 
\begin{align}
	\rho_{\pi}^\lambda(\vx) = \lambda {\rho}_{\mathrm{N}}(\vx) + (1-\lambda)\rho_\pi(\vx).
\end{align}
Here, $0 < \lambda < 1$ is a hyper-parameter, $\pi$ is a policy to be learned by maximizing the risk, $g: \mathcal{X} \mapsto \mathbb{R}$ is a classifier to be learned by minimizing the risk, and $\ell_{\mathrm{sym}}: \mathbb{R} \mapsto \mathbb{R}$ is a \emph{symmetric loss} satisfying 
\begin{align}
	\ell_{\mathrm{sym}}(g(\vx)) + \ell_{\mathrm{sym}}(-g(\vx)) = c, 
\end{align}
for all $\vx \in \mathcal{X}$, where $c \in \mathbb{R}$ is a constant. 
Appropriate choices of the hyper-parameter and loss will be discussed in Sections~\ref{section:lambda} and~\ref{section:loss}, respectively.  

We note that the balanced risk assumes that  the positive and negative class priors are equal to $\frac{1}{2}$. 
This assumption typically makes the balanced risk more restrictive than other risks, because a classifier is learned to maximize the balanced accuracy instead of the accuracy~\citep{MenonEtAl2013,LuNMS19}.
However, the balanced risk is not too restrictive for IL, because the goal is to learn the expert policy and the classifier is discarded after learning.
Moreover, existing methods such as GAIL can be viewed as methods that optimize the balanced risk, as will be discussed in Section~\ref{section:lambda}. 

Next, we prove that the optimization in Eq.~\eqref{eq:rail_obj} yields the expert policy under the following assumption. 

\begin{assumption}[Mixture state-action density]
	\label{assumption_1}
	The state-action density of the learning policy $\pi$ is a mixture of the state-action densities of the expert and non-expert policies with a mixing coefficient $0 \leq \kappa(\pi) \leq 1$:
	\begin{align}
		\rho_{\pi}(\vx) = \kappa(\pi) \rho_{\mathrm{E}}(\vx) + (1-\kappa(\pi)) \rho_{\mathrm{N}}(\vx),
		\label{eq:interpolating policy}
	\end{align}
	where $\rho_{\pi}(\vx)$, $\rho_{\mathrm{E}}(\vx)$, and $\rho_{\mathrm{N}}(\vx)$ are the state-action densities of the learning policy, the expert policy, and the non-expert policy, respectively. 
\end{assumption}
This assumption is based on the following observation on a typical optimization procedure of $\pi$: At the start of learning, $\pi$ is randomly initialized and generates data that are similar to those of the non-expert policy. This scenario corresponds to Eq.~\eqref{eq:interpolating policy} with $\kappa(\pi) \approx 0$.
As training progresses, the policy improves and generates data that are a mixture of those from the expert and non-expert policies. This scenario corresponds to Eq.~\eqref{eq:interpolating policy} with $0 < \kappa(\pi) < 1$. 
Indeed, the scenario where the agent successfully learns the expert policy corresponds to Eq.~\eqref{eq:interpolating policy} with $\kappa(\pi) = 1$. 

We note that a policy uniquely corresponding to $\rho_{\pi}$ in Eq.~\eqref{eq:interpolating policy} is  a mixture between $\pi_{\mathrm{E}}$ and $\pi_{\mathrm{N}}$ with a mixture coefficient depending on $\kappa(\pi)$. 
However, we cannot directly evaluate the value of $\kappa(\pi)$. This is because we do not directly optimize the state-action density $\rho_{\pi}$. Instead, we optimize the policy $\pi$ by using an RL method, as will be discussed in Section~\ref{section:pseudo_label}. 


Under Assumption~\ref{assumption_1}, we obtain Lemma~\ref{theorem:risk}.
\begin{lemma}
	\label{theorem:risk}
	Letting $\ell_{\mathrm{sym}}(\cdot)$ be a symmetric loss that satisfies $\ell_{\mathrm{sym}}(g(\vx)) + \ell_{\mathrm{sym}}(-g(\vx)) = c$, $\forall \vx\in\mathcal{X}$ and a constant $c \in \mathbb{R}$, the following equality holds.
	\begin{align}
		\mathcal{R}(g; \rho^\prime, \rho_{\pi}^\lambda, \ell_{\mathrm{sym}}) 
		&= (\alpha - \kappa(\pi)(1-\lambda) ) \mathcal{R}(g ; \rho_{\mathrm{E}}, \rho_{\mathrm{N}}, \ell_{\mathrm{sym}}) \notag \\
		&\phantom{=} + \frac{1-\alpha+\kappa(\pi)(1-\lambda)}{2} c.
	\end{align}
\end{lemma}
The proof is given in Appendix~\ref{appendix:proof}, which follows~\cite{charoenphakdee19a}. 
This lemma indicates that, a minimizer $g^\star$ of $\mathcal{R}(g; \rho^\prime, \rho_{\pi}^\lambda, \ell_{\mathrm{sym}})$ is identical to that of $\mathcal{R}(g ; \rho_{\mathrm{E}}, \rho_{\mathrm{N}}, \ell_{\mathrm{sym}})$: 
\begin{align} 
	g^\star &= \argmin_{g} \mathcal{R}(g; \rho^\prime, \rho_{\pi}^\lambda, \ell_{\mathrm{sym}}) \notag \\
	&= \argmin_{g}\mathcal{R}(g ; \rho_{\mathrm{E}}, \rho_{\mathrm{N}}, \ell_{\mathrm{sym}}), \label{eq:opt_g}
\end{align} 
when $\alpha - \kappa(\pi)(1-\lambda) > 0$.
This result enables us to prove that the maximizer of the  risk optimization in Eq.~\eqref{eq:rail_obj} is the expert policy.
\begin{theorem}
	\label{theorem:optimality}
	Given the optimal classifier $g^\star$ in Eq.~\eqref{eq:opt_g}, the solution of $\max_{\pi} \mathcal{R}(g^\star; \rho^\prime, \rho_{\pi}^\lambda, \ell_{\mathrm{sym}})$ is equivalent to the expert policy.
\end{theorem}
\begin{sproof}
	It can be shown that the solution of $\max_{\pi} \mathcal{R}(g^\star; \rho^\prime, \rho_{\pi}^\lambda, \ell_{\mathrm{sym}})$ is equivalent to the solution of $\max_{\pi} \kappa(\pi) \Big( \mathbb{E}_{\rho_{\mathrm{E}}}\left[ \ell_{\mathrm{sym}}(-g^\star(\vx)) \right] - \mathbb{E}_{\rho_{\mathrm{N}}}\left[ \ell_{\mathrm{sym}}(-g^\star(\vx)) \right] \Big)$.
	Since $g^\star = \argmin_{g}\mathcal{R}(g ; \rho_{\mathrm{E}}, \rho_{\mathrm{N}}, \ell_{\mathrm{sym}})$, it follows that $\mathbb{E}_{\rho_{\mathrm{E}}}\left[ \ell_{\mathrm{sym}}(-g^\star(\vx)) \right] - \mathbb{E}_{\rho_{\mathrm{N}}}\left[ \ell_{\mathrm{sym}}(-g^\star(\vx)) \right] > 0$. 
	Thus, $\max_{\pi} \kappa(\pi) \Big( \mathbb{E}_{\rho_{\mathrm{E}}}\left[ \ell_{\mathrm{sym}}(-g^\star(\vx)) \right] - \mathbb{E}_{\rho_{\mathrm{N}}}\left[ \ell_{\mathrm{sym}}(-g^\star(\vx)) \right] \Big)$ is solved by increasing $\kappa(\pi)$ to $1$. 
	Because $\kappa(\pi)=1$ if and only if $\rho_\pi(\vx) = \rho_{\mathrm{E}}(\vx)$ under Assumption~\ref{assumption_1}, we conclude that the solution of $\max_{\pi} \mathcal{R}(g^\star; \rho^\prime, \rho_{\pi}^\lambda, \ell_{\mathrm{sym}})$ is equivalent to the expert policy.
\end{sproof}
A detailed proof is given in Appendix~\ref{appendix:proof}.
This result indicates that robust IL can be achieved by optimizing the risk in Eq.~\eqref{eq:rail_obj}. 
In a practice aspect, this is a significant advance compared to the prior work~\citep{Wu2019}, because Theorem~\ref{theorem:optimality} shows that robust IL can be achieved \emph{without} the knowledge of the mixing coefficient $\alpha$ or additional labels to estimate $\alpha$. 
Next, we present a new IL method that empirically solves Eq.~\eqref{eq:rail_obj} by using co-pseudo-labeling. 

\subsection{Co-pseudo-labeling for Risk Optimization}	
\label{section:pseudo_label}

While the risk in Eq.~\eqref{eq:rail_obj} leads to robust IL, we cannot directly optimize this risk in our setting. 
This is because the risk contains an expectation over $\rho_{\mathrm{N}}(\vx)$, but we are given only demonstration samples drawn from $\rho'(\vx)$\footnote{The expectation over $\rho_{\pi}(\vx)$ can be approximated using trajectories independently collected by the policy $\pi$.}. 
We address this issue by using co-pseudo-labeling to approximately draw samples from $\rho_{\mathrm{N}}(\vx)$, as described below.

Recall that the optimal classifier $g^\star(\vx)$ in Eq.~\eqref{eq:opt_g} also minimizes the risk $\mathcal{R}(g ; \rho_{\mathrm{E}}, \rho_{\mathrm{N}}, \ell_{\mathrm{sym}})$. 
Therefore, given a state-action sample $\widetilde{\vx} \in \mathcal{X}$, we can use $g^\star(\widetilde{\vx})$ to predict whether $\tilde{\vx}$ is drawn from $\rho_{\mathrm{E}}$ or $\rho_{\mathrm{N}}$. 
Specifically,  $\tilde{\vx}$ is predicted to be drawn from $\rho_{\mathrm{E}}$ when $g^\star(\tilde{\vx}) \geq 0$, and it is predicted to be drawn from $\rho_{\mathrm{N}}$ when $g^\star(\tilde{\vx}) < 0$. 
Based on this observation, our key idea is to approximate the expectation over $\rho_{\mathrm{N}}$ in Eq.~\eqref{eq:rail_obj} by using samples that are predicted by $g$ to be drawn from $\rho_{\mathrm{N}}$.


To realize this idea, we firstly consider the following {empirical risk} with a semi-supervised learning technique called \emph{pseudo-labeling}~\citep{Chapelle2010}:
\begin{align}
	\widehat{\mathcal{R}}(g) &= \frac{1}{2}\widehat{\mathbb{E}}_{\mathcal{D}}\left[ \ell_{\mathrm{sym}}(g(\vx))\right]  + \frac{\lambda}{2}\widehat{\mathbb{E}}_{\mathcal{P}}\left[ \ell_{\mathrm{sym}}(-g(\vx))\right] \notag \\
	&\phantom{=} + \frac{1-\lambda}{2}\widehat{\mathbb{E}}_{\mathcal{B}}\left[ \ell_{\mathrm{sym}}(-g(\vx))\right]. \label{eq:p_risk} 
\end{align}
Here, $\widehat{\mathbb{E}}[\cdot]$ denotes an empirical expectation (i.e., the sample average), $\mathcal{D}$ is the demonstration dataset in Eq.~\eqref{eq:data_1}, $\mathcal{B}$ is a dataset of trajectories collected by using $\pi$, and $\mathcal{P}$ is a dataset of pseudo-labeled demonstrations obtained by choosing demonstrations in $\mathcal{D}$ with $g(\vx) < 0$, i.e., samples that are predicted to be drawn from $\rho_{\mathrm{N}}$. 
This risk with pseudo-labeling enables us to empirically solve Eq.~\eqref{eq:rail_obj} in our setting.
However, the trained classifier may perform poorly, mainly because samples in $\mathcal{P}$ are labeled by the classifier itself. 
Specifically, the classifier during training may predict the labels of demonstrations incorrectly, i.e., demonstrations drawn from $\rho_{\mathrm{E}}(\vx)$ are incorrectly predicted to be drawn from $\rho_{\mathrm{N}}(\vx)$. This may degrade the performance of the classifier, because using incorrectly-labeled data can reinforce the classifier to be over-confident in its incorrect prediction~\citep{Kingma2014}.


To remedy the over-confidence of the classifier, we propose \emph{co-pseudo-labeling}, which combines the ideas of pseudo-labeling and co-training~\citep{blum98}.
Specifically, we train two classifiers denoted by $g_1$ and $g_2$ by minimizing the following empirical risks:
\begin{align}
	\widehat{\mathcal{R}}_1(g_1) &= \frac{1}{2}\widehat{\mathbb{E}}_{\mathcal{D}_1}\left[ \ell_{\mathrm{sym}}(g_1(\vx))\right]  + \frac{\lambda}{2}\widehat{\mathbb{E}}_{\mathcal{P}_1}\left[ \ell_{\mathrm{sym}}(-g_1(\vx))\right] \notag \\
	&\phantom{=} + \frac{1-\lambda}{2}\widehat{\mathbb{E}}_{\mathcal{B}}\left[ \ell_{\mathrm{sym}}(-g_1(\vx))\right],  \label{eq:r_1} \\
	\widehat{\mathcal{R}}_2(g_2) &= \frac{1}{2}\widehat{\mathbb{E}}_{\mathcal{D}_2}\left[ \ell_{\mathrm{sym}}(g_2(\vx))\right]  + \frac{\lambda}{2}\widehat{\mathbb{E}}_{\mathcal{P}_2}\left[ \ell_{\mathrm{sym}}(-g_2(\vx))\right] \notag \\
	&\phantom{=} + \frac{1-\lambda}{2}\widehat{\mathbb{E}}_{\mathcal{B}}\left[ \ell_{\mathrm{sym}}(-g_2(\vx))\right],  \label{eq:r_2}
\end{align}
where $\mathcal{D}_1$ and $\mathcal{D}_2$ are disjoint subsets of $\mathcal{D}$. 
Pseudo-labeled dataset $\mathcal{P}_1$ is obtained by choosing demonstrations from $\mathcal{D}_2$ with $g_2(\vx) < 0$, and pseudo-labeled dataset $\mathcal{P}_2$ is obtained by choosing demonstrations from $\mathcal{D}_1$ with $g_1(\vx) < 0$.
With these risks, we reduce the influence of over-confident classifiers because $g_1$ is trained using samples pseudo-labeled by $g_2$ and vice-versa~\citep{HanYYNXHTS18}. 
We call our proposed method \emph{Robust IL with Co-pseudo-labeling} (RIL-Co). 

We implement RIL-Co by using a stochastic gradient method to optimize the empirical risk where we alternately optimize the classifiers and policy. 
Recall from Eq.~\eqref{eq:rail_obj} that we aim to maximize $\mathcal{R}(g; \rho^\prime, \rho_{\pi}^\lambda, \ell_{\mathrm{sym}})$ w.r.t.~$\pi$. 
After ignoring terms that are constant w.r.t.~$\pi$, solving this maximization is equivalent to maximizing $\mathbb{E}_{\rho_{\pi}}\left[ \ell_{\mathrm{sym}}(-g(\vx))\right]$ w.r.t.~$\pi$. 
This objective is identical to the RL objective in Eq.~\eqref{eq:rl_obj2} with a reward function $r(\vx) =\ell_{\mathrm{sym}}(-g(\vx))$. Therefore, we can train the policy in RIL-Co by simply using an existing RL method, e.g., the trust-region policy gradient~\citep{Wu2017}. 
We summarize the procedure of RIL-Co in Algorithm~\ref{algo:rail}.
Sourcecode of our implementation is available at \textcolor{Blue}{\nolinkurl{https://github.com/voot-t/ril_co}}.

\begin{algorithm*}[ht!]
	\caption{RIL-Co: Robust Imitation Learning with Co-pseudo-labeling}
	\label{algo:rail}
	\begin{algorithmic}[1]
		\STATE \textbf{Input: } Demonstration dataset $\mathcal{D}$, initial policy $\pi$, and initial classifiers $g_1$ and $g_2$.
		\STATE Set hyper-parameter $\lambda=0.5$ (see Section~\ref{section:lambda}) and batch-sizes ($B=U=V=640$ and $K=128$).
		\STATE Split $\mathcal{D}$ into two disjoint datasets $\mathcal{D}_1$ and $\mathcal{D}_2$.
		\WHILE{ Not converge }
		\WHILE { $| \mathcal{B} | < B$ with batch size $B$ }	
		\STATE Use $\pi$ to collect and include transition samples into $\mathcal{B}$
		\ENDWHILE
		\STATE \textbf{Co-pseudo-labeling:} 
		\begin{ALC@g}
		\STATE Sample $\{ \vx_{u} \}_{u=1}^U$ from $\mathcal{D}_2$, and choose $K$ samples with $g_2(\vx_u) < 0$  in an ascending order as $\mathcal{P}_1$.
		\STATE Sample $\{ \vx_{v} \}_{v=1}^V$ from $\mathcal{D}_1$, and choose $K$ samples with $g_1(\vx_v) < 0$  in an ascending order as $\mathcal{P}_2$.
		\end{ALC@g}
		\STATE \textbf{Train classifiers:} 
		\begin{ALC@g}
		\STATE Train $g_1$ by performing gradient descent to minimize the empirical risk $\widehat{R}_1(g_1)$ using $\mathcal{D}_1$, $\mathcal{P}_1$ and $\mathcal{B}$.
		\STATE Train $g_2$ by performing gradient descent to minimize the empirical risk $\widehat{R}_2(g_2)$ using $\mathcal{D}_2$, $\mathcal{P}_2$ and $\mathcal{B}$.

		\end{ALC@g}
		\STATE \textbf{Train policy:} 
		\begin{ALC@g}
		\STATE Train the policy by an RL method with transition samples in $\mathcal{B}$ and rewards $r(\vx)=\ell_{\mathrm{sym}}(-g_1(\vx))$.

		\end{ALC@g}

		\ENDWHILE 
	\end{algorithmic}
\end{algorithm*}

\begin{figure*}[ht!]
\centering
	\begin{minipage}[c]{0.56\linewidth}
		\centering
		\captionof{table}{Examples of losses and their symmetric property, i.e., whether $\ell(z) + \ell(-z) = c$. We denote normalized counterparts of non-symmetric losses by (N). The AP loss in Eq.~\eqref{eq:apl} is a linear combination of the normalized logistic and sigmoid losses. }
		\label{table:loss}
		\bgroup
		\def\arraystretch{1.4}
		\begin{tabular}{ l || c | c }
			\hline
			Loss name & $\ell(z)$ 						& Symmetric 	 \\ 
			\hline \hline
			Logistic 	& $\log(1+\exp(-z))$  			&  \ding{55}		\\
			Hinge 		& $\max(1-z, 0)$  				&  \ding{55}		\\
			\hline 
			Sigmoid 	&  $\begin{aligned}1/({1+\exp(z)})\end{aligned}$    &  \ding{51}  	     	\\ 	
			Unhinged    &  $1-z$     					&  \ding{51}  \\ 
			Logistic (N) &  
			$\begin{aligned}\frac{\log(1+\exp(-z))}{\Sigma_{k \in \{-1, 1\}} \log(1+\exp(-zk))}	\end{aligned}$			   					&  \ding{51}  		\\ 
			Hinge (N) &  
			$\begin{aligned}\frac{\max(1-z, 0)}{\Sigma_{k \in \{-1, 1\}} \max(1-zk, 0)}	\end{aligned}$			   					&  \ding{51}  		\\ 
		\end{tabular}
		\egroup
	\end{minipage}
\hfill 
\centering
	\begin{minipage}[c]{0.42\linewidth}
		\centering
		\includegraphics[width=0.90\linewidth]{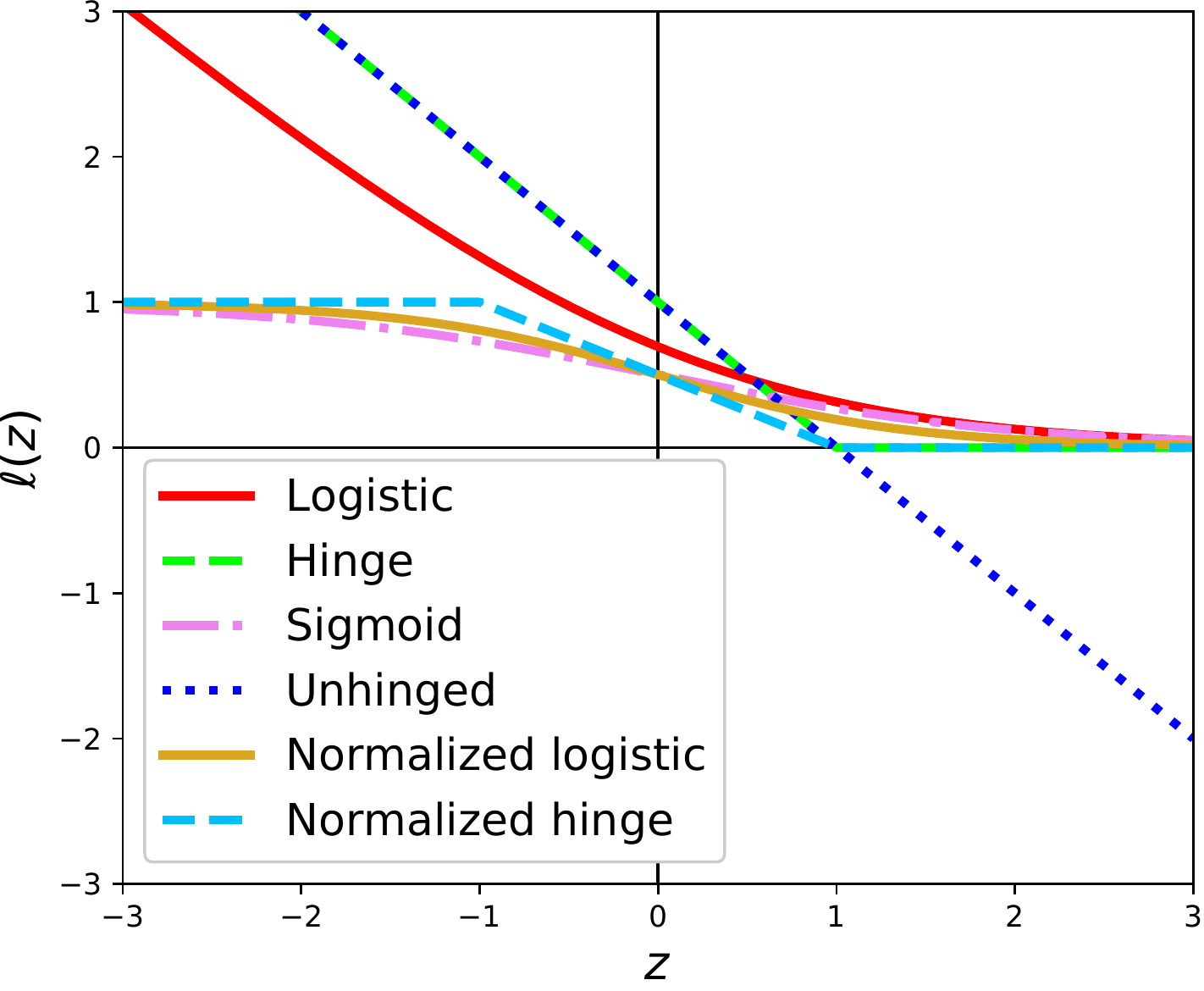} 
		\caption{The value of losses in Table~\ref{table:loss}. Non-symmetric losses, i.e., the logistic and hinge losses, become symmetric after normalization.}
		\label{figure:loss}
	\end{minipage}

	\vspace{-2mm}

\end{figure*}

\subsection{Choice of Hyper-parameter} \label{section:lambda}

We propose to use $\lambda=0.5$ for RIL-Co, because it makes the equality in Eq.~\eqref{eq:opt_g} holds which is essential for our theoretical result in Section~\ref{section:rail}. 
Specifically, recall that Theorem~\ref{theorem:optimality} relies on the equality in Eq.~\eqref{eq:opt_g}. At the same time, the equality in Eq.~\eqref{eq:opt_g} holds if the the following inequality holds:
\begin{align}
\alpha - \kappa(\pi)(1-\lambda) > 0. \label{eq:inequality}
\end{align} 
This inequality depends on $\alpha$, $\kappa(\pi)$, and $\lambda$, where $\alpha$ is unknown, $\kappa(\pi)$ depends on the policy, and $\lambda$ is the hyper-parameter. 
However, we cannot choose nor evaluate $\kappa(\pi)$ since we do not directly optimize $\kappa(\pi)$ during policy training. 
Thus, we need to choose an appropriate value of $\lambda$ so that the inequality in Eq.~\eqref{eq:inequality} always holds.
Recall that we assumed $0.5 < \alpha < 1$ in Section~\ref{section:noise_data}.
Under this assumption, the inequality in Eq.~\eqref{eq:inequality} holds regardless of the true value of $\alpha$ when $\kappa(\pi)(1-\lambda) \leq 0.5$. 
Since the value of $\kappa(\pi)$ increases to 1 as the policy improves by training (see Assumption~\ref{assumption_1}), the appropriate value of $\lambda$ is $0.5 \leq \lambda < 1$.

However, a large value of $\lambda$ may not be preferable in practice since it increases the influence of pseudo-labels on the risks (i.e., the second term in Eqs.~\eqref{eq:r_1} and~\eqref{eq:r_2}). These pseudo-labels should not have larger influence than real labels (i.e., the third term in Eqs.~\eqref{eq:r_1} and~\eqref{eq:r_2}). For this reason, we decided to use $\lambda = 0.5$, which is the smallest value of $\lambda$ that ensures the inequality in Eq.~\eqref{eq:inequality} to be always held during training. 
Nonetheless, we note that while RIL-Co with $\lambda=0.5$ already yields good performance in the following experiments, the performance may still be improved by fine-tuning $\lambda$ using e.g., grid-search. 

\textbf{Remarks. }
Choosing $\lambda=0$ corresponds to omitting co-pseudo-labeling, and doing so reduces RIL-Co to variants of GAIL which are not robust. 
Concretely, the risk optimization problem of RIL-Co with $\lambda=0$ is $\max_{\pi} \min_{g} \mathcal{R}(g; \rho^\prime, \rho_{\pi}, \ell_{\mathrm{sym}})$.
By using the logistic loss: $\ell(z) = \log(1+\exp(-z))$, instead of a symmetric loss, we obtain the following risk:
\begin{align}
2\mathcal{R}(g; \rho^\prime, \rho_{\pi}, \ell) 
&= \mathbb{E}_{\rho^\prime}[ \log ( 1 + \exp(-g(\vx)) ) ] \notag \\
&\phantom{=}+ \mathbb{E}_{\rho_\pi}[ \log ( 1 + \exp(g(\vx)) ) ],
\end{align}
which is the negative of GAIL's objective in Eq.~\eqref{eq:gail}\footnote{Here, $\rho^\prime$ replaces $\rho_\mathrm{E}$. The sign flips since RIL-Co and GAIL solve max-min and min-max problems, respectively.}. 
Meanwhile, we may obtain other variants of GAIL by using summetric losses such as the sigmoid loss and the unhinged loss~\citep{Rooyen2015}. 
In particular, with the unhinged loss: $\ell(z) = 1-z$, the risk becomes the negative of Wasserstein GAIL's objective with an additive constant~\citep{LiSE17,Xiao2019}:
\begin{align}
	2\mathcal{R}(g; \rho^\prime, \rho_{\pi}, \ell) 
	&\!=\! \mathbb{E}_{\rho^\prime}[ -g(\vx) ] + \mathbb{E}_{\rho_\pi}[ g(\vx) ] + \frac{1}{2}.
\end{align}
However, even when $\ell(z)$ is symmetric, we conjecture that such variants of GAIL are not robust, because $\lambda=0$ does not make the inequality in Eq.~\eqref{eq:inequality} holds when $\kappa(\pi) > 0.5$. 

\subsection{Choice of Symmetric Loss}	\label{section:loss}

In our implementation of RIL-Co, we use the active-passive loss (AP loss)~\citep{ma2020normalized} defined as
\begin{align}
	\ell_{\mathrm{AP}}(z) 
	&= \frac{ 0.5 \times \log (1+\exp(-z)) } {\log (1+\exp(-z)) + \log (1+\exp(z))} \notag \\
	&\phantom{=}  + \frac{0.5}{1 + \exp(z)}, \label{eq:apl}
\end{align}
which satisfies $\ell_{\mathrm{AP}}(z) + \ell_{\mathrm{AP}}(-z) = 1$.
This loss is a linear combination of two symmetric losses: the normalized logistic loss (the first term) and the sigmoid loss (the second term). It was shown that this loss suffers less from the issue of under-fitting when compared to each of the normalized logistic loss or the sigmoid loss~\citep{ma2020normalized}.
However, we emphasize that any symmetric loss can be used to learn the expert policy with RIL-Co, as indicated by our theoretical result in Section~\ref{section:rail}. 
In addition, any loss can be made symmetric by using normalization~\citep{ma2020normalized}. 
Therefore, the requirement of symmetric losses is not a severe limitation. Table~\ref{table:loss} and Figure~\ref{figure:loss} show examples of non-symmetric and symmetric losses.

\vspace{-1mm}

\section{EXPERIMENTS}

\vspace{-1mm}

We evaluate the robustness of RIL-Co on continuous-control benchmarks simulated by PyBullet simulator (HalfCheetah, Hopper, Walker2d, and Ant)~\citep{coumans2019}. 
These tasks are equipped with the true reward functions that we use for the evaluation purpose. 
The learning is conducted using the true states (e.g., joint positions) and not the visual observations. 
We report the mean and standard error of the performance (cumulative true rewards) over 5 trials. 

\begin{figure*}[ht!]
	
	\centering
	\includegraphics[width=0.95\linewidth]{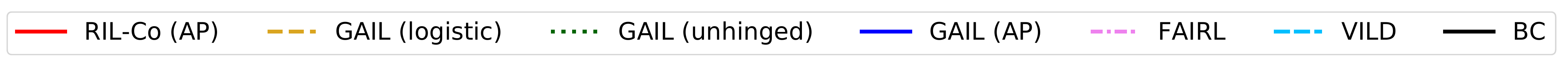}

	\vspace{0.5mm}

	\includegraphics[width=0.24\linewidth]{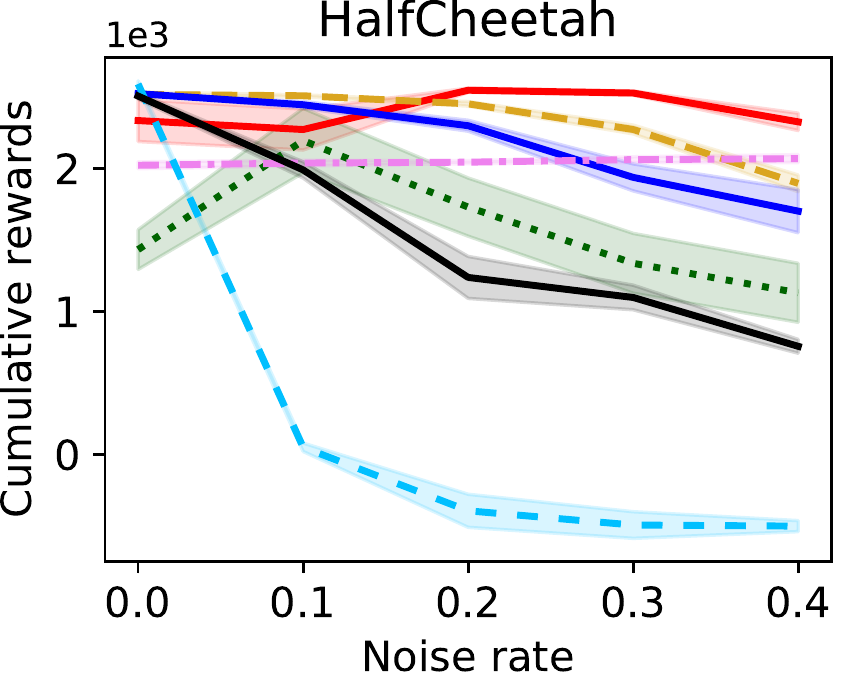} \hfill
	\includegraphics[width=0.25\linewidth]{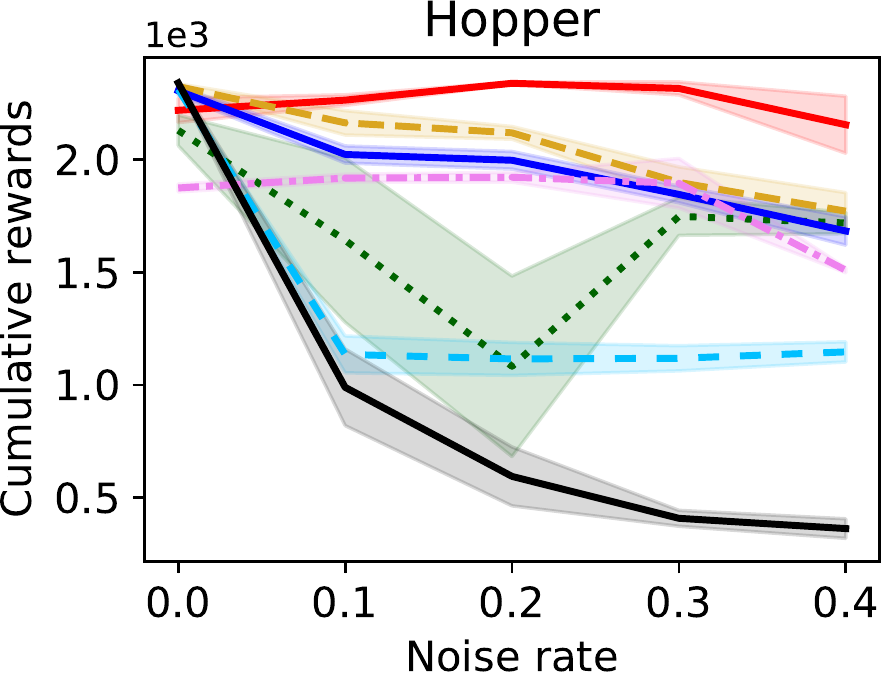} \hfill
	\includegraphics[width=0.25\linewidth]{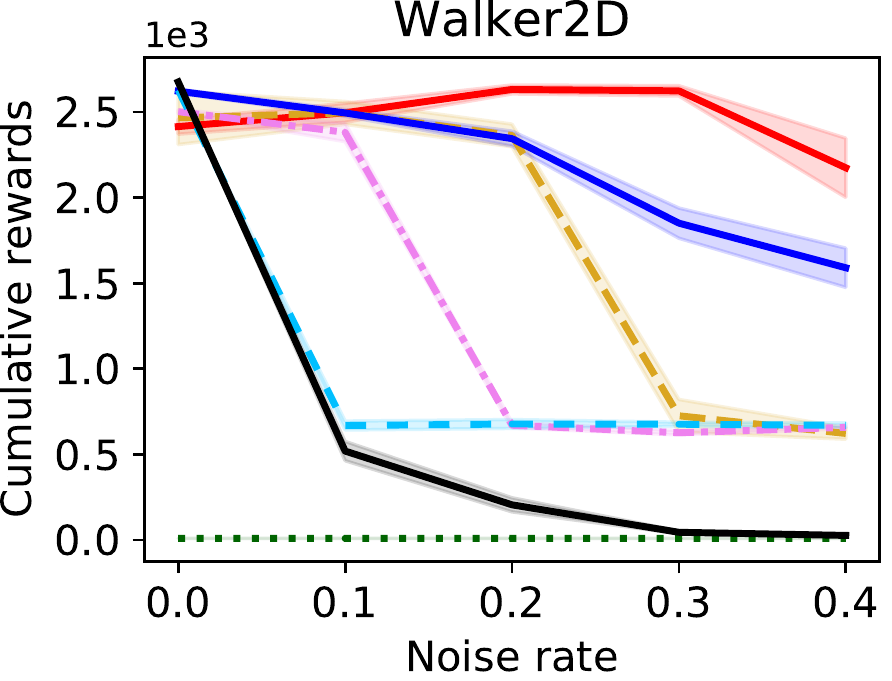} \hfill
	\includegraphics[width=0.24\linewidth]{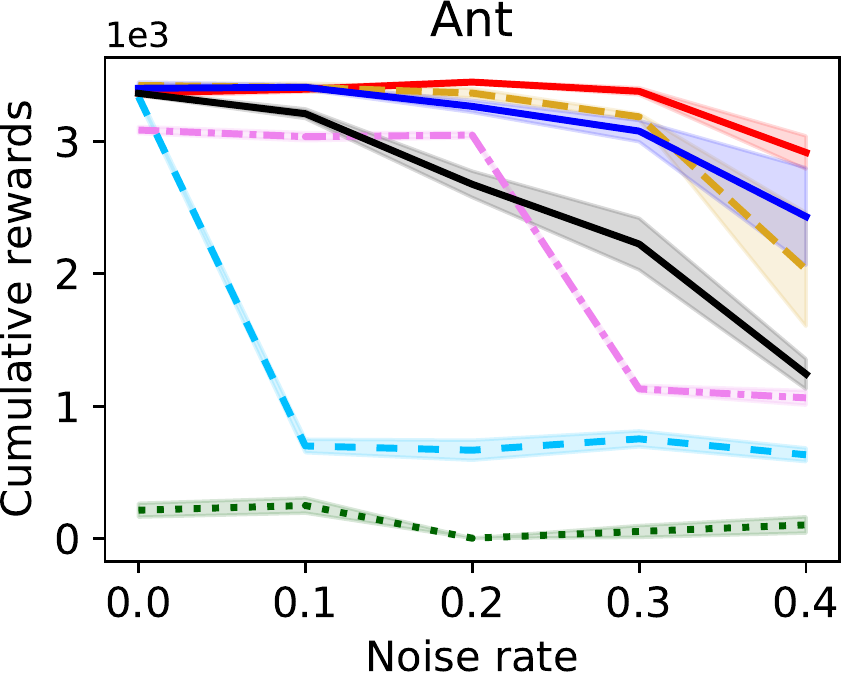} \hfill

	\vspace{-5mm}
	\caption{Final performance in continuous-control benchmarks with different noise rates. Vertical axes denote cumulative rewards obtained during the last 1000 training iterations. Shaded regions denote standard errors computed over 5 runs. RIL-Co  performs well even when the noise rate increases. Meanwhile, the performance of other methods significantly degrades as the noise rate increases. 
	}	
	\label{figure:exp_main}

	\vspace{1mm}

	\includegraphics[width=0.24\linewidth]{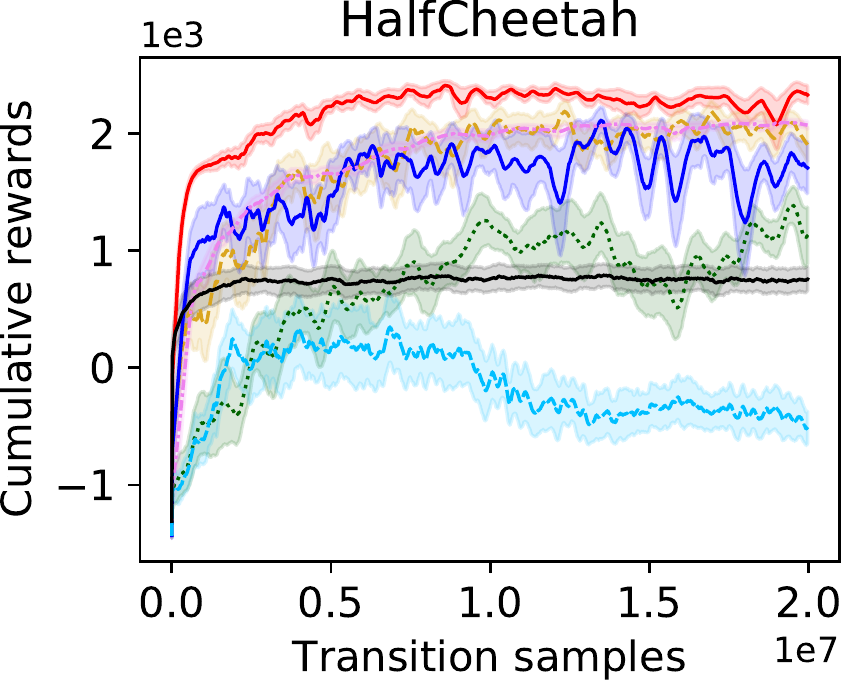} \hfill
	\includegraphics[width=0.25\linewidth]{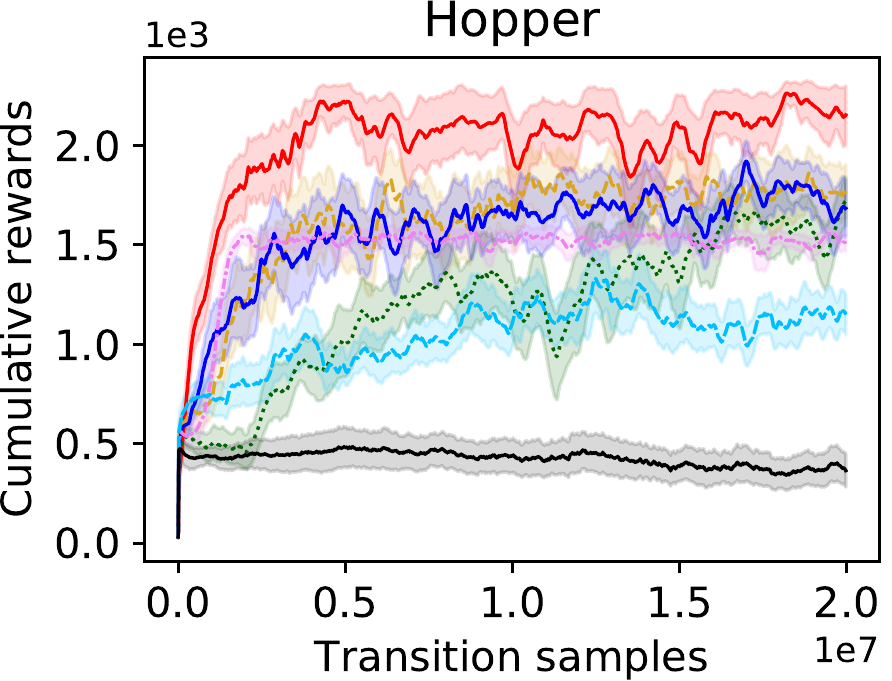} \hfill
	\includegraphics[width=0.245\linewidth]{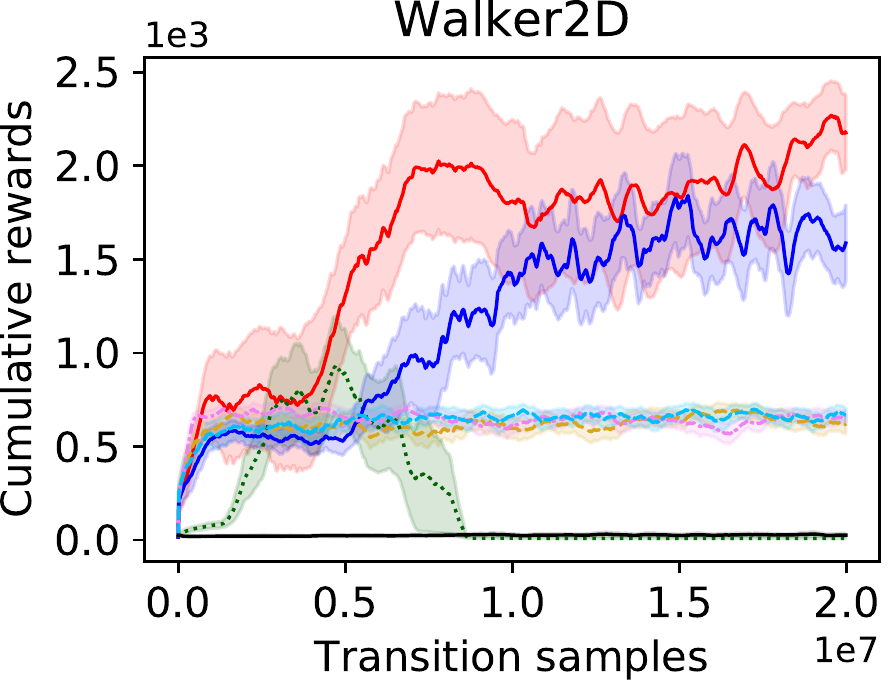} \hfill
	\includegraphics[width=0.245\linewidth]{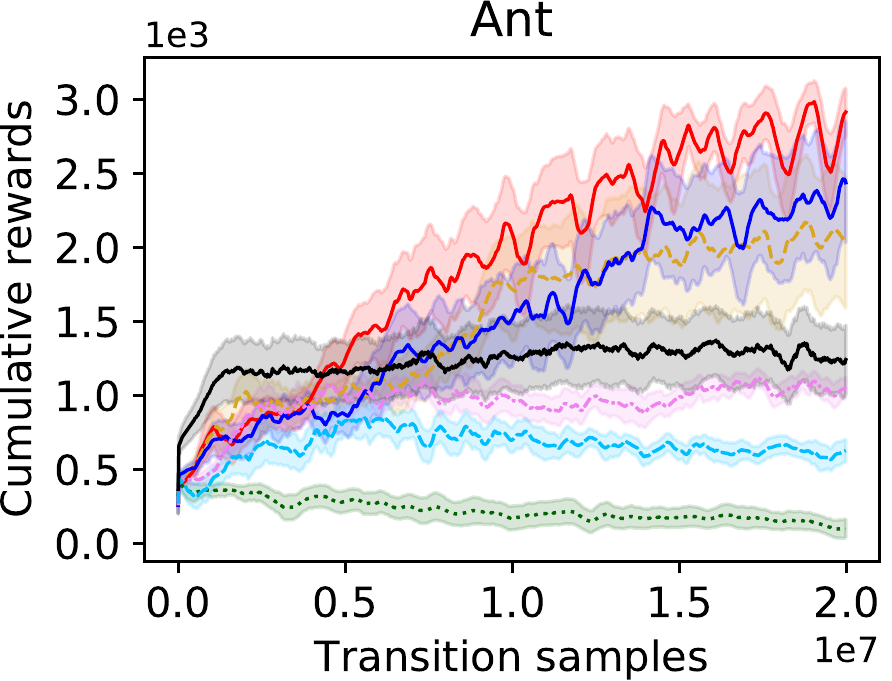} \hfill

	\vspace{-5mm}
	\caption{Performance against the number of transition samples in continuous-control benchmarks with noise rate $\delta = 0.4$. RIL-Co achieves better performances and uses less transition samples compared to other methods. 
	}	
	\label{figure:exp_main2}

	\vspace{-3mm}

\end{figure*}

We compare RIL-Co with the AP loss against the following baselines: BC~\citep{Pomerleau88}, FAIRL~\citep{ghasemipour20a}, VILD~\citep{TangkarattEtAl2020}, and three variants of GAIL where each variant uses different losses: logistic, unhinged, and AP. 
As discussed in the remarks of Section~\ref{section:lambda}, GAIL with the logistic loss denotes the original GAIL that performs density matching with the Jensen-Shannon divergence, GAIL with the unhinged loss denotes a variant of GAIL that performs density matching with the Wasserstein distance, and GAIL with the AP loss corresponds to RIL-Co without co-pseudo-labeling. 


All methods use policy networks with 2 hidden-layers of 64 hyperbolic tangent units. We use similar networks with 100 hyperbolic tangent units for classifiers in RIL-Co and discriminators in other methods. 
The policy networks are trained by the trust region policy gradient~\citep{Wu2017} from a public implementation~\citep{pytorchrl}. The classifiers and discriminators are trained by Adam~\citep{KingmaBa2014} with the gradient penalty regularizer with the regularization parameter of 10~\citep{GulrajaniAADC17}. The total number of transition samples collected by the learning policy is 20 million.
More details of experimental setting can be found in Appendix~\ref{appendix:exp_settings}.

\vspace{-2mm}

\subsection{Evaluation on Noisy Datasets with Different Noise Rates}
\label{section:exp_main}

\vspace{-1mm}

In this experiment, we evaluate RIL-Co on noisy datasets generated with different noise rates. 
To obtain datasets, we firstly train policies by RL with the true reward functions. Next, we choose 6 policy snapshots where each snapshot is trained using different numbers of transition samples. Then, we use the best performing policy snapshot (in terms of cumulative rewards) to collect $10000$ expert state-action samples, and use the other 5 policy snapshots to collect a total of $10000$ non-expert state-action samples.
Lastly, we generate datasets with different noise rates by mixing expert and non-expert state-action samples, where noise rate $\delta \in \{ 0, 0.1, 0.2, 0.3, 0.4\}$ approximately determines the number of randomly chosen non-expert state-action samples. 
Specifically, a dataset consisting of $10000$ expert samples corresponds to a dataset with $\delta=0$ (i.e., no noise), whereas a dataset consisting of $10000$ expert samples and $7500$ randomly chosen non-expert samples corresponds to a dataset with $\delta=0.4$ approximately\footnote{The true noise rates of these datasets are as follows: \\$\tilde{\delta} \in \{0, 1000/11000, 2500/12500, 5000/15000, 7500/17500\}$.}. 
We note that the value of $\delta$ approximately equals to the value of $1-\alpha$ in Eq.~\eqref{eq:data_2}. 


Figure~\ref{figure:exp_main} shows the final performance achieved by each method. 
We can see that RIL-Co outperforms comparison methods and achieves the best performance in high noise scenarios where $\delta \in \{ 0.2, 0.3, 0.4\}$. 
Meanwhile, in low noise scenarios where $\delta \in \{ 0.0, 0.1\}$, RIL-Co performs comparable to the best performing methods such as GAIL with the logistic and AP losses. 
Overall, the results show that RIL-Co achieves good performance in the presence of noises, while the other methods fail to learn and their performance degrades as the noise rate increases.

\begin{figure*}[ht!]
	\includegraphics[width=0.19\linewidth]{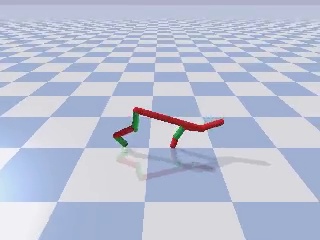} \hfill
	\includegraphics[width=0.19\linewidth]{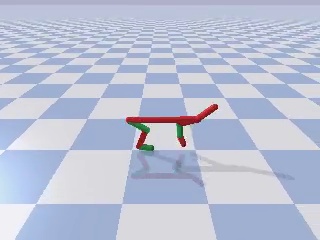} \hfill
	\includegraphics[width=0.19\linewidth]{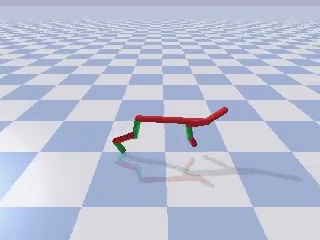} \hfill
	\includegraphics[width=0.19\linewidth]{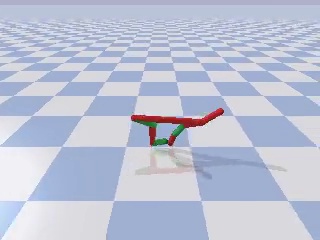} \hfill
	\includegraphics[width=0.19\linewidth]{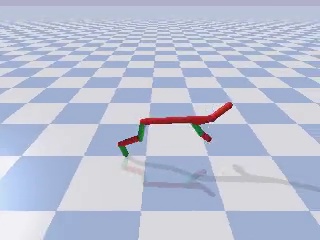} \hfill

	\includegraphics[width=0.19\linewidth]{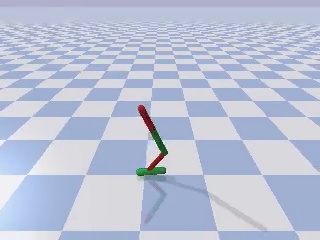} \hfill
	\includegraphics[width=0.19\linewidth]{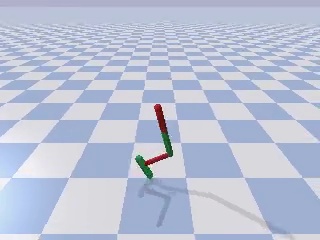} \hfill
	\includegraphics[width=0.19\linewidth]{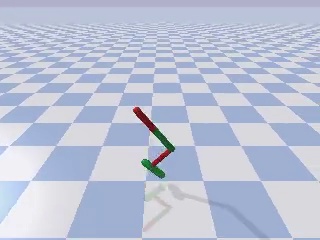} \hfill
	\includegraphics[width=0.19\linewidth]{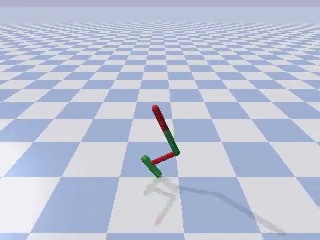} \hfill
	\includegraphics[width=0.19\linewidth]{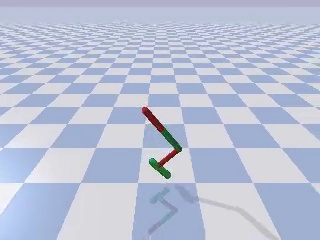} \hfill

	\includegraphics[width=0.19\linewidth]{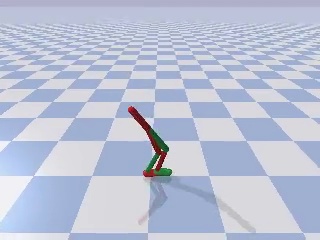} \hfill
	\includegraphics[width=0.19\linewidth]{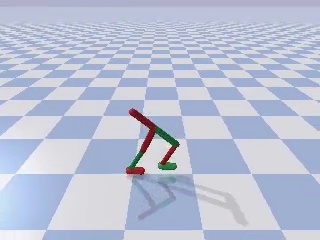} \hfill
	\includegraphics[width=0.19\linewidth]{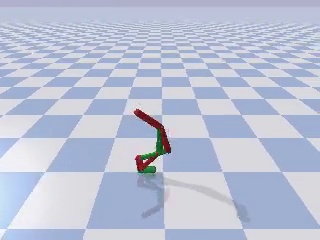} \hfill
	\includegraphics[width=0.19\linewidth]{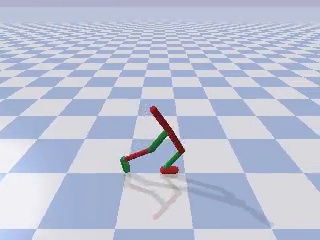} \hfill
	\includegraphics[width=0.19\linewidth]{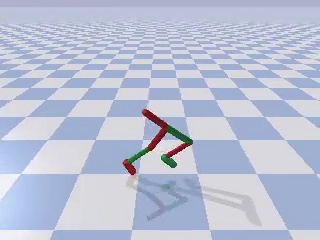} \hfill

	\includegraphics[width=0.19\linewidth]{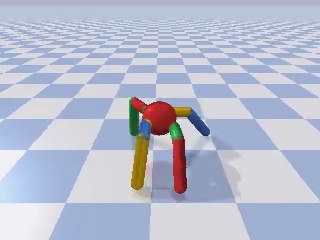} \hfill
	\includegraphics[width=0.19\linewidth]{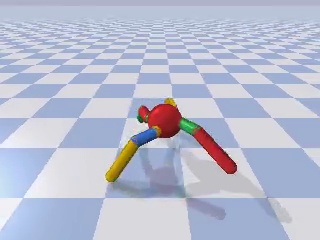} \hfill
	\includegraphics[width=0.19\linewidth]{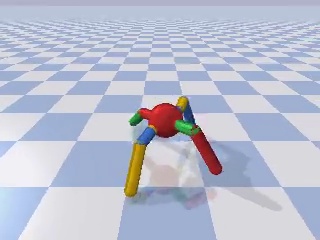} \hfill
	\includegraphics[width=0.19\linewidth]{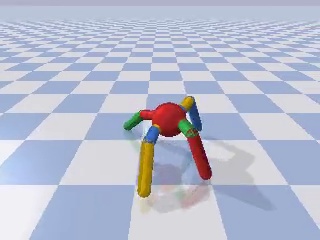} \hfill
	\includegraphics[width=0.19\linewidth]{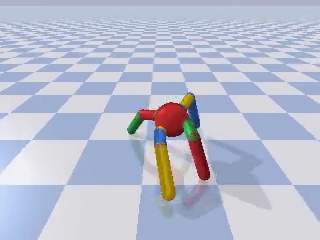} \hfill

	\caption{Visualization of the first 100 time steps of trajectories obtained by RIL-Co from datasets with noise rate $\delta=0.4$. Time step increases from the leftmost figure ($t=0)$ to the rightmost figure ($t=100)$. 
	RIL-Co agents successfully solve these tasks. The obtained trajectories also closely resemble expert demonstrations. 
	}	
	\label{figure:snapshot_rilco}

	\vspace{-1mm}

\end{figure*}

In contrast, density matching methods, namely FAIRL and GAIL with the logistic and unhinged losses, do not perform well. 
This is as expected, because these methods would learn a policy that is a mixture of the expert and non-expert policies and would not perform well. 
Also notice that GAIL with the unhinged loss performs very poorly on the Walker2D and Ant tasks even with noiseless demonstrations where $\delta=0$.
This is an intriguing result given that the unhinged loss is also symmetric similarly to the AP loss. 
We conjecture that the poor performance is due to the unboundedness from below of the unhinged loss (see Figure~\ref{figure:loss}). 
This unboundedness may lead to a poorly behaved classifier that outputs values with very large magnitudes, as suggested by~\cite{charoenphakdee19a}. 
With such a classifier, we expect that GAIL with the unhinged loss would require a strong regularization to perform well, especially for complex control tasks. 

In addition, we can see that RIL-Co is more robust when compared to GAIL with the AP loss. This result supports our theorem which indicates that a symmetric loss alone is insufficient for robustness. 
Interestingly, with expert demonstrations (i.e., $\delta=0$), GAIL tends to outperform RIL-Co. 
This is perhaps because co-pseudo-labeling introduces additional biases. This could be avoided by initially using $\lambda=0$ (i.e., performing GAIL) and gradually increasing the value to $\lambda=0.5$ as learning progresses.

On the other hand, VILD performs poorly with noisy datasets even with a small noise rate of $\delta=0.1$. 
We conjecture that this is because VILD could not accurately estimate the noise distributions due to the violation of its Gaussian assumption.
Specifically, VILD assumes that noisy demonstrations are generated by adding Gaussian noise to actions drawn from the expert policy, and that expert demonstrations consist of low-variance actions. 
However, noisy demonstrations in this experiment are generated by using policy snapshots without adding any noise.
In this case, non-expert demonstrations may consist of low-variance actions (e.g., non-expert policy may yield a constant action). 
Due to this, VILD cannot accurately estimate the noise distributions and performs poorly. 
Meanwhile, behavior cloning (BC) does not perform well. This is because BC assumes that demonstrations are generated by experts. It also suffers from the issue of compounding error which worsens the performance.

\begin{figure*}[ht!]
	
	\centering
	\includegraphics[width=0.70\linewidth]{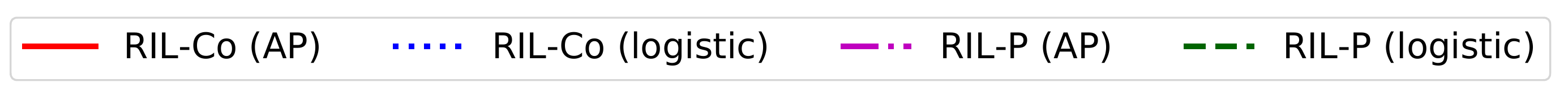}

	\vspace{1mm}

	\includegraphics[width=0.245\linewidth]{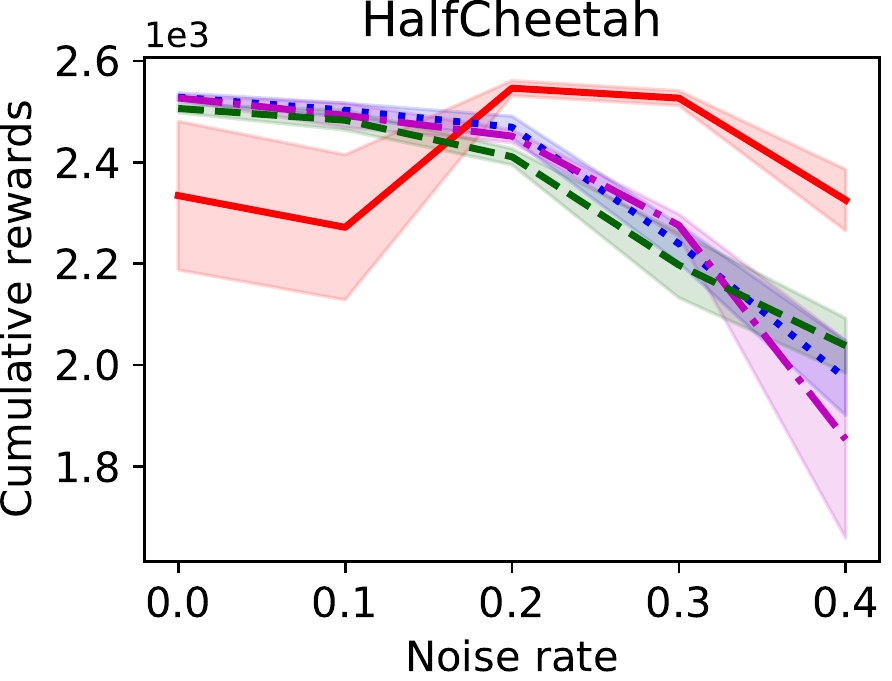} \hfill
	\includegraphics[width=0.245\linewidth]{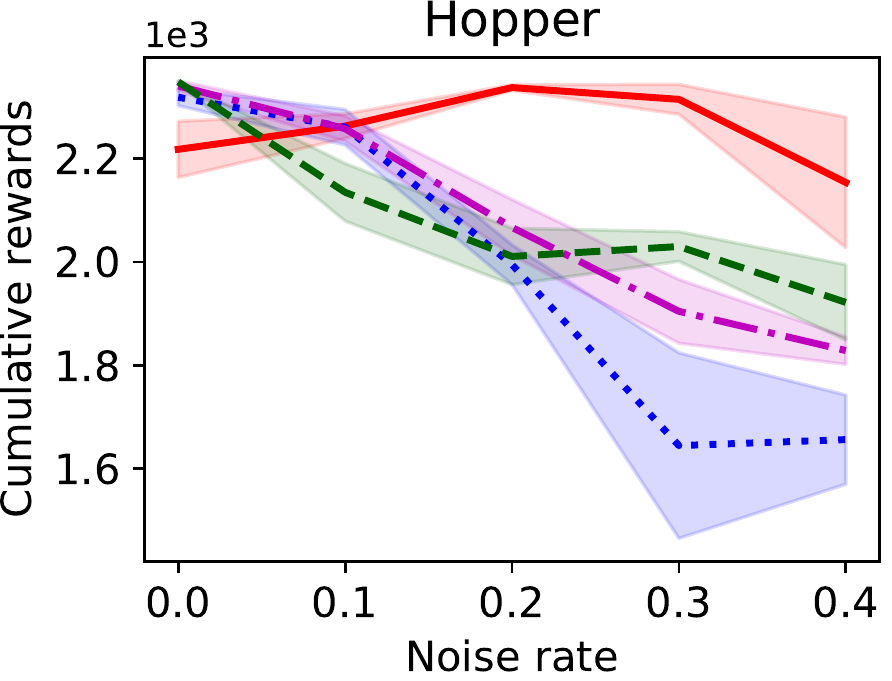} \hfill
	\includegraphics[width=0.245\linewidth]{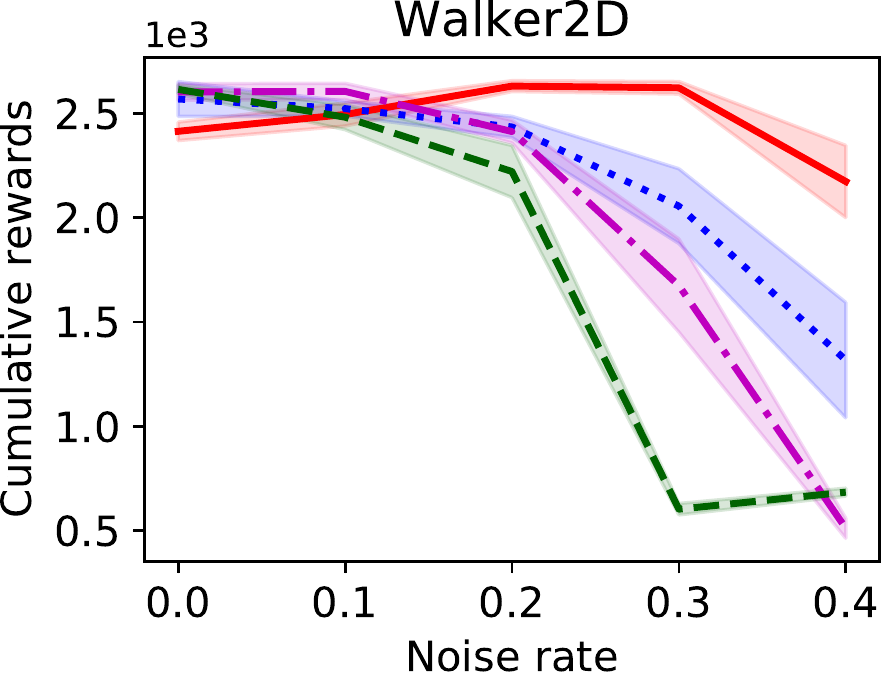} \hfill
	\includegraphics[width=0.245\linewidth]{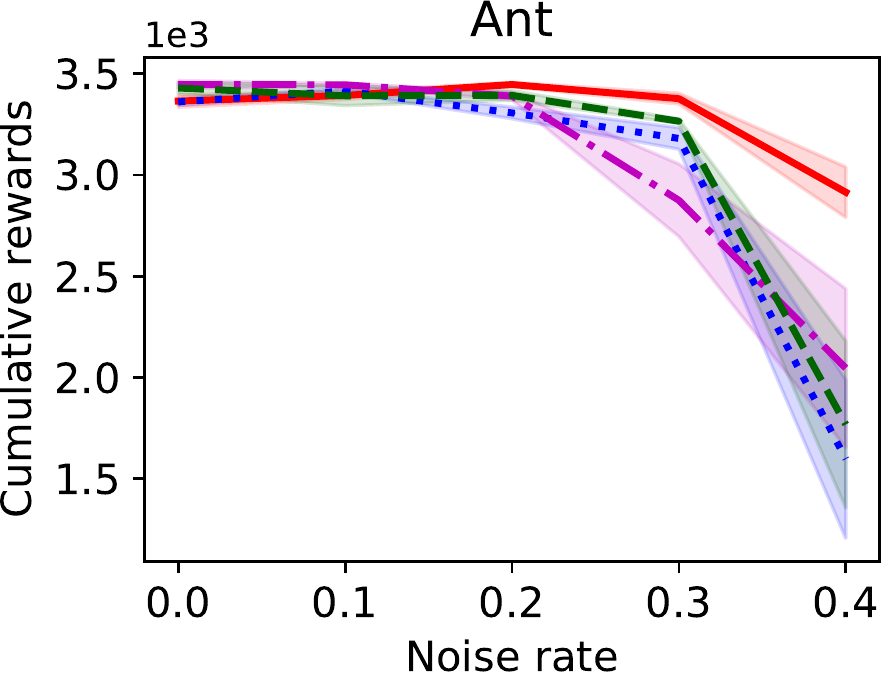} \hfill

	\vspace{-5mm}
	\caption{Final performance of variants of RIL-Co in the ablation study. RIL-Co with the AP loss performs the best. This shows that using a symmetric loss and co-pseudo-labeling together is important for robustness. 
	}	
	\label{figure:exp_ablation_noise}

	\vspace{-2mm}

\end{figure*}
	
Next, Figure~\ref{figure:exp_main2} shows the performance against the number of transition samples collected by the learning policy for $\delta=0.4$. 
RIL-Co achieves better performances and uses less transition samples when compared to other methods.
This result indicates that RIL-CO is data efficient. 
Results for different noise rates can be found in Appendix~\ref{appendix:add_results}, and they show a similar tendency.

Lastly, Figure~\ref{figure:snapshot_rilco} depicts visualization of trajectories obtained by the policy of RIL-Co. 
Indeed, RIL-Co agents successfully solve the tasks. The trajectories also closely resemble expert demonstrations shown in Appendix~\ref{appendix:exp_settings}. 
This qualitative result further verifies that RIL-Co successfully learns the expert policy. 

\subsection{Ablation Study}
\label{section:exp_ablation}

In this section, we conduct ablation study by evaluating different variants of RIL-Co.
Specifically, we evaluate RIL-Co with the logistic loss to investigate the importance of symmetric loss.
In addition, to investigate the importance of co-pseudo-labeling, we also evaluate Robust IL with Pseudo-labeling (RIL-P) which uses naive pseudo-labeling in Eq.~\eqref{eq:p_risk} instead of co-pseudo-labeling.
Experiments are conducted using the same datasets in the previous section. 

The result in Figure~\ref{figure:exp_ablation_noise} shows that RIL-Co with the AP loss outperforms the variants.  
This result further indicates that using a symmetric loss and co-pseudo-labeling together is important for robustness.

\subsection{Evaluation on Gaussian Noise Dataset}
\label{section:exp_gaussian}

\begin{figure}[t] 
	\centering
	\includegraphics[width=0.80\linewidth]{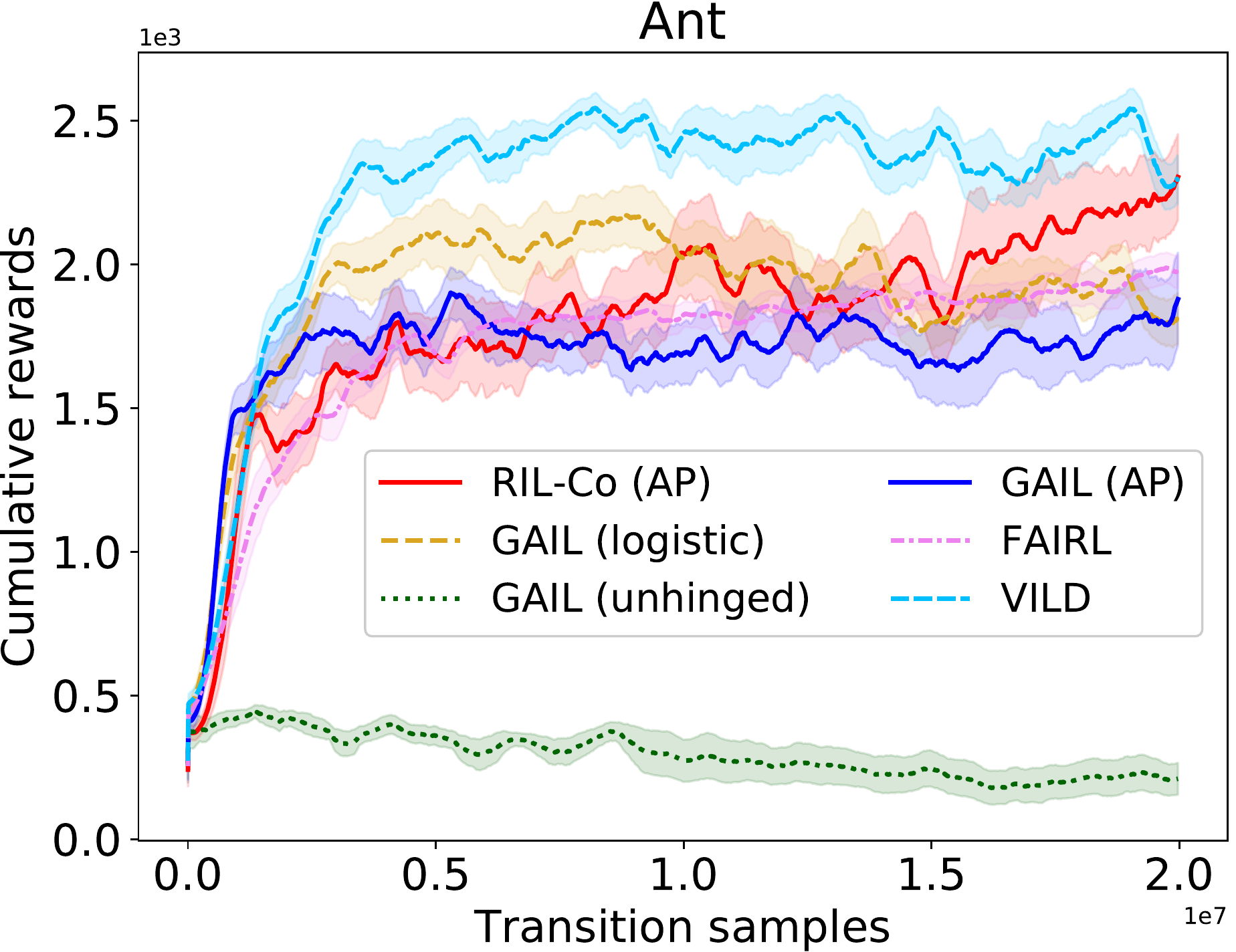} 
	\caption{Performance with a Gaussian noise dataset. RIL-Co performs better than others except VILD.}
	\label{figure:exp_gaussian}

	\vspace{-2mm}

\end{figure}

Next, we evaluate RIL-Co with the AP loss in the Ant task with a noisy dataset generated by Gaussian noise.
Specifically, we use a dataset with 10000 expert and 7500 non-expert state-action samples, where non-expert samples are obtained by adding Gaussian noise to action samples drawn from the expert policy. 
This dataset is generated according to the main assumption of VILD~\citep{TangkarattEtAl2020}, and we expect VILD to perform well in this experiment.

Figure~\ref{figure:exp_gaussian} depicts the performance against the number of transition samples.
It can be seen that VILD performs much better in this setting compared to the previous setting in Figure~\ref{figure:exp_main}. 
This is as expected, since the Gaussian assumption of VILD is correct in this setting but incorrect in the previous setting.
Still, RIL-Co achieves a performance comparable to that of VILD with 20 million samples, even though RIL-Co relies on a milder data generation assumption (see Section~\ref{section:noise_data}). 
Meanwhile, the other methods do not perform as well as RIL-Co and VILD.

Overall, the empirical results in our experiments indicate that RIL-Co is more robust against noisy demonstrations when compared to existing methods.

\section{CONCLUSIONS}

\vspace{-1mm}

We presented a new method for IL from noisy demonstrations. 
We proved that robust IL can be achieved by optimizing a classification risk with a symmetric loss, and we proposed RIL-Co which optimizes the risk by using co-pseudo-labeling. 
We showed through experiments that RIL-Co is more robust against noisy demonstrations when compared to existing methods. 

In this paper, we utilized co-pseudo-labeling to approximate data from non-expert densities. 
However, data from non-expert densities may be readily available~\citep{GrollmanB11}. 
Utilizing such data for robust IL is an interesting future direction.

\section*{Acknowledgement}
We thank the reviewers for their useful comments. 
NC was supported by MEXT scholarship and Google PhD Fellowship program.
MS was supported by KAKENHI 17H00757.

\bibliography{my_lib.bib}






\clearpage
\onecolumn

\appendix
\section{PROOFS} 
\label{appendix:proof}

\setcounter{assumption}{0}
\setcounter{lemma}{0}
\setcounter{theorem}{0}

\subsection{Proof of Lemma 1}
We firstly restate the main assumption of our theoretical result:
\begin{assumption}[Mixture state-action density]
	The state-action density of the learning policy $\pi$ is a mixture of the state-action densities of the expert and non-expert policies with a mixing coefficient $0 \leq \kappa(\pi) \leq 1$:
	\begin{align}
		\rho_{\pi}(\vx) = \kappa(\pi) \rho_{\mathrm{E}}(\vx) + (1-\kappa(\pi)) \rho_{\mathrm{N}}(\vx),
	\end{align}
	where $\rho_{\pi}(\vx)$, $\rho_{\mathrm{E}}(\vx)$, and $\rho_{\mathrm{N}}(\vx)$ are the state-action densities of the learning policy, the expert policy, and the non-expert policy, respectively. 
\end{assumption}
Under this assumption and the assumption that $\rho'(\vs,\va) = \alpha \rho_\mathrm{E}(\vs,\va) + (1-\alpha) \rho_{\mathrm{N}}(\vs,\va)$, we obtain Lemma~1.
\begin{lemma}
	Letting $\ell_{\mathrm{sym}}(\cdot)$ be a symmetric loss that satisfies $\ell_{\mathrm{sym}}(g(\vx)) + \ell_{\mathrm{sym}}(-g(\vx)) = c$, $\forall \vx\in\mathcal{X}$ and a constant $c \in \mathbb{R}$, the following equality holds.
	\begin{align}
		\mathcal{R}(g; \rho^\prime, \rho_{\pi}^\lambda, \ell_{\mathrm{sym}}) 
		&= (\alpha - \kappa(\pi)(1-\lambda) ) \mathcal{R}(g ; \rho_{\mathrm{E}}, \rho_{\mathrm{N}}, \ell_{\mathrm{sym}}) + \frac{1-\alpha+\kappa(\pi)(1-\lambda)}{2} c. \label{eq:app_1}
	\end{align}
\end{lemma}

\begin{proof}
Firstly, we define $\tilde{\kappa}(\pi,\lambda) = \kappa(\pi)(1-\lambda)$ and $\delta^{\ell}(\vx) = \ell(g(\vx)) + \ell(-g(\vx))$.
Then, we substitute $\rho^\prime(\vx) := \alpha \rho_{\mathrm{E}}(\vx) + (1-\alpha) \rho_{\mathrm{N}}(\vx)$ and $\rho_{\pi}(\vx) = \kappa(\pi) \rho_{\mathrm{E}}(\vx) + (1-\kappa(\pi)) \rho_{\mathrm{N}}(\vx)$ into the risk $\mathcal{R}(g; \rho^\prime, \rho_{\pi}^\lambda, \ell)$.
\begin{align}
	2\mathcal{R}(g; \rho^\prime, \rho_{\pi}^\lambda, \ell) 
	&= \mathbb{E}_{\rho^\prime}\left[ \ell(g(\vx))\right] 
	+ \mathbb{E}_{\rho_{\pi}^{\lambda}}\left[ \ell(-g(\vx))\right] \notag \\
	&= \mathbb{E}_{\rho^\prime}\left[ \ell(g(\vx))\right] 
	+ (1-\lambda)\mathbb{E}_{\rho_\pi}\left[ \ell(-g(\vx))\right] 
	+ \lambda\mathbb{E}_{{\rho}_{\mathrm{N}}}\left[ \ell(-g(\vx))\right] \notag \\
	&= \alpha\mathbb{E}_{\rho_{\mathrm{E}}}\left[ \ell(g(\vx))\right] 
	+ (1-\alpha)\mathbb{E}_{\rho_{\mathrm{N}}}\left[ \ell(g(\vx))\right] 
	\notag \\
	&\phantom{=} + (\kappa(\pi)(1-\lambda))\mathbb{E}_{\rho_{\mathrm{E}}}\left[ \ell(-g(\vx))\right] 	
	+ (1-\kappa(\pi))(1-\lambda)\mathbb{E}_{\rho_{\mathrm{N}}}\left[ \ell(-g(\vx))\right]
	+ \lambda\mathbb{E}_{\rho_{\mathrm{N}}}\left[ \ell(-g(\vx))\right] \notag \\
	&= \alpha\mathbb{E}_{\rho_{\mathrm{E}}}\left[ \ell(g(\vx))\right] 
	+ (1-\alpha)\mathbb{E}_{\rho_{\mathrm{N}}}\left[ \ell(g(\vx))\right] 
	+ \tilde{\kappa}(\pi,\lambda)\mathbb{E}_{\rho_{\mathrm{E}}}\left[ \ell(-g(\vx))\right] 
	+ (1 - \tilde{\kappa}(\pi,\lambda))\mathbb{E}_{\rho_{\mathrm{N}}}\left[ \ell(-g(\vx))\right] \notag \\
	&= \alpha\mathbb{E}_{\rho_{\mathrm{E}}}\left[ \ell(g(\vx))\right] 
	+ (1-\alpha)\mathbb{E}_{\rho_{\mathrm{N}}}\left[ \delta^{\ell}(\vx) - \ell(-g(\vx))\right] \notag \\
	&\phantom{=} 
	+ \tilde{\kappa}(\pi,\lambda)\mathbb{E}_{\rho_{\mathrm{E}}}\left[ \delta^{\ell}(\vx) - \ell(g(\vx))\right] 
	+ (1 - \tilde{\kappa}(\pi,\lambda))\mathbb{E}_{\rho_{\mathrm{N}}}\left[ \ell(-g(\vx))\right] \notag \\
	&= \alpha\mathbb{E}_{\rho_{\mathrm{E}}}\left[ \ell(g(\vx))\right] 
	+ (1-\alpha)\mathbb{E}_{\rho_{\mathrm{N}}}\left[ \delta^{\ell}(\vx) \right] 
	- \mathbb{E}_{\rho_{\mathrm{N}}}\left[\ell(-g(\vx))\right] + \alpha \mathbb{E}_{\rho_{\mathrm{N}}}\left[\ell(-g(\vx))\right] \notag \\
	&\phantom{=} 
	+ \tilde{\kappa}(\pi,\lambda)\mathbb{E}_{\rho_{\mathrm{E}}}\left[ \delta^{\ell}(\vx) \right] 
	- \tilde{\kappa}(\pi,\lambda)\mathbb{E}_{\rho_{\mathrm{E}}}\left[\ell(g(\vx))\right] 
	+ \mathbb{E}_{\rho_{\mathrm{N}}}\left[ \ell(-g(\vx))\right] 
	- \tilde{\kappa}(\pi,\lambda)\mathbb{E}_{\rho_{\mathrm{N}}}\left[ \ell(-g(\vx))\right] \notag \\
	&= (\alpha - \tilde{\kappa}(\pi,\lambda)) \left( \mathbb{E}_{\rho_{\mathrm{E}}}\left[ \ell(g(\vx)) \right] + \mathbb{E}_{\rho_{\mathrm{N}}}\left[ \ell(-g(\vx))\right] \right)
	+ (1-\alpha)\mathbb{E}_{\rho_{\mathrm{N}}}\left[ \delta^{\ell}(\vx) \right] 
	+ \tilde{\kappa}(\pi,\lambda)\mathbb{E}_{\rho_{\mathrm{E}}}\left[ \delta^{\ell}(\vx) \right] \notag \\
	&= 2 (\alpha - \tilde{\kappa}(\pi,\lambda) ) \mathcal{R}(g ; \rho_{\mathrm{E}}, \rho_{\mathrm{N}}, \ell)
	+ (1-\alpha)\mathbb{E}_{\rho_{\mathrm{N}}}\left[ \delta^{\ell}(\vx) \right] 
	+ \tilde{\kappa}(\pi,\lambda)\mathbb{E}_{\rho_{\mathrm{E}}}\left[ \delta^{\ell}(\vx) \right].
\end{align}
For symmetric loss, we have $\delta^{\ell_{\mathrm{sym}}}(\vx) = \ell_{\mathrm{sym}}(g(\vx)) + \ell_{\mathrm{sym}}(-g(\vx))=c$ for a constant $c\in\mathbb{R}$.
With this, we can express the left hand-side of Eq.~\eqref{eq:app_1} as follows:
\begin{align}
	\mathcal{R}(g; \rho^\prime, \rho_{\pi}^\lambda, \ell_{\mathrm{sym}})
	&= (\alpha - \kappa(\pi)(1-\lambda)) \mathcal{R}(g ; \rho_{\mathrm{E}}, \rho_{\mathrm{N}}, \ell_{\mathrm{sym}})
	+ \frac{(1-\alpha)}{2}\mathbb{E}_{\rho_{\mathrm{N}}}\left[ \delta^{\ell_{\mathrm{sym}}}(\vx) \right] 
	+ \frac{\kappa(\pi)(1-\lambda)}{2}\mathbb{E}_{\rho_{\mathrm{E}}}\left[ \delta^{\ell_{\mathrm{sym}}}(\vx) \right] \notag  \\
	&= (\alpha - \kappa(\pi)(1-\lambda)) \mathcal{R}(g ; \rho_{\mathrm{E}}, \rho_{\mathrm{N}}, \ell_{\mathrm{sym}})
	+ \frac{(1-\alpha)}{2}\mathbb{E}_{\rho_{\mathrm{N}}}\left[ c \right] 
	+ \frac{\kappa(\pi)(1-\lambda)}{2}\mathbb{E}_{\rho_{\mathrm{E}}}\left[ c \right] \notag  \\
	&= (\alpha - \kappa(\pi)(1-\lambda)) \mathcal{R}(g ; \rho_{\mathrm{E}}, \rho_{\mathrm{N}}, \ell_{\mathrm{sym}})
	+ \frac{1-\alpha+\kappa(\pi)(1-\lambda)}{2} c.
\end{align}
This equality concludes the proof of Lemma~1. 
Note that this proof follows~\citet{charoenphakdee19a}. 
\end{proof}

\subsection{Proof of Theorem 1}
Lemma 1 indicates that, when $\alpha - \kappa(\pi)(1-\lambda) > 0$, we have
\begin{align} 
	g^\star &= \argmin_{g} \mathcal{R}(g; \rho^\prime, \rho_{\pi}^\lambda, \ell_{\mathrm{sym}}) \notag \\
	&= \argmin_{g}\mathcal{R}(g ; \rho_{\mathrm{E}}, \rho_{\mathrm{N}}, \ell_{\mathrm{sym}}). \label{eq:app_2}
\end{align} 
With this, we obtain Theorem~1 which we restate and prove below.
\begin{theorem}
	Given the optimal classifier $g^\star$ in Eq.~\eqref{eq:app_2}, the solution of $\max_{\pi} \mathcal{R}(g^\star; \rho^\prime, \rho_{\pi}^\lambda, \ell_{\mathrm{sym}})$ is equivalent to the expert policy.
\end{theorem}
\begin{proof}
		By using the definitions of the risk and $\rho_{\pi}^\lambda(\vx)$, $\mathcal{R}(g^\star; \rho^\prime, \rho_{\pi}^\lambda, \ell_{\mathrm{sym}})$ can be expressed as
		\begin{align}
			\mathcal{R}(g^\star; \rho^\prime, \rho_{\pi}^\lambda, \ell_{\mathrm{sym}})
			&= \frac{1}{2}\mathbb{E}_{\rho^\prime}\! \left[ \ell_{\mathrm{sym}}(g^\star(\vx))\right] + \frac{\lambda}{2}\mathbb{E}_{\rho_{\mathrm{N}}}\! \left[ \ell_{\mathrm{sym}}(-g^\star(\vx))\right] + \frac{1-\lambda}{2}\mathbb{E}_{\rho_{\pi}}\! \left[ \ell_{\mathrm{sym}}(-g^\star(\vx))\right].
		\end{align}
		Since the first and second terms are constant w.r.t. $\pi$, the solution of $\max_{\pi} \mathcal{R}(g^\star; \rho^\prime, \rho_{\pi}^\lambda, \ell_{\mathrm{sym}})$ is equivalent to the solution of 
		$\max_{\pi} \mathbb{E}_{\rho_{\pi}} \left[ \ell_{\mathrm{sym}}(-g^\star(\vx))\right]$, where we omit the positive constant factor $(1-\lambda)/2$.
		Under Assumption~1 which assumes $\rho_{\pi}(\vx) = \kappa(\pi) \rho_{\mathrm{E}}(\vx) + (1-\kappa(\pi)) \rho_{\mathrm{N}}(\vx)$ , we can further express the objective function as 
		\begin{align}
			\mathbb{E}_{\rho_{\pi}}\left[ \ell_{\mathrm{sym}}(-g^\star(\vx)) \right] 
			&= \kappa(\pi) \mathbb{E}_{\rho_{\mathrm{E}}}\left[ \ell_{\mathrm{sym}}(-g^\star(\vx)) \right] + (1 - \kappa(\pi)) \mathbb{E}_{\rho_{\mathrm{N}}}\left[ \ell_{\mathrm{sym}}(-g^\star(\vx)) \right] \notag \\
			&= \kappa(\pi) \Big( \mathbb{E}_{\rho_{\mathrm{E}}}\left[ \ell_{\mathrm{sym}}(-g^\star(\vx)) \right] - \mathbb{E}_{\rho_{\mathrm{N}}}\left[ \ell_{\mathrm{sym}}(-g^\star(\vx)) \right] \Big) + \mathbb{E}_{\rho_{\mathrm{N}}}\left[ \ell_{\mathrm{sym}}(-g^\star(\vx)) \right].
		\end{align}
		The last term is a constant w.r.t.~$\pi$ and can be safely ignored.
		The right hand-side is maximized by increasing $\kappa(\pi)$ to 1 when the inequality  $\mathbb{E}_{\rho_{\mathrm{E}}}\left[ \ell_{\mathrm{sym}}(-g^\star(\vx)) \right] - \mathbb{E}_{\rho_{\mathrm{N}}}\left[ \ell_{\mathrm{sym}}(-g^\star(\vx)) \right] > 0$ holds. 
		Since $g^\star$ is also the optimal classifier of $\mathcal{R}(g; \rho_{\mathrm{E}}, \rho_{\mathrm{N}}, \ell_{\mathrm{sym}})$, the inequality $\mathbb{E}_{\rho_{\mathrm{E}}}\left[ \ell_{\mathrm{sym}}(-g^\star(\vx)) \right] - \mathbb{E}_{\rho_{\mathrm{N}}}\left[ \ell_{\mathrm{sym}}(-g^\star(\vx)) \right] > 0$ holds.
		Specifically, the expected loss of classifying expert data as non-expert: $\mathbb{E}_{\rho_{\mathrm{E}}}\left[ \ell_{\mathrm{sym}}(-g^\star(\vx)) \right]$, is larger to the expected loss of classifying non-expert data as non-expert:  $\mathbb{E}_{\rho_{\mathrm{N}}}\left[ \ell_{\mathrm{sym}}(-g^\star(\vx)) \right]$. 
		Thus, the objective can only be maximized by increasing $\kappa(\pi)$ to $1$. 
		Because $\kappa(\pi)=1$ if and only if $\rho_\pi(\vx) = \rho_{\mathrm{E}}(\vx)$, we conclude that the solution of $\max_{\pi} \mathcal{R}(g^\star; \rho^\prime, \rho_{\pi}^\lambda, \ell_{\mathrm{sym}})$ is equivalent to $\pi_{\mathrm{E}}$.
\end{proof}

\section{DATASETS AND IMPLEMENTATION} 
\label{appendix:exp_settings}

We conduct experiments on continuous-control benchmarks simulated by PyBullet simulator~\citep{coumans2019}.
We consider four locomotion tasks, namely \texttt{HalfCheetah}, \texttt{Hopper}, \texttt{Walker2d}, and \texttt{Ant}, where the goal is to control the agent to move forward to the right.
We use true states of the agents and do not use visual observation. 
To obtain demonstration datasets, we collect expert and non-expert state-action samples by using 6 policy snapshots trained by the trust-region policy gradient method (ACKTR)~\citep{Wu2017}, where each snapshot is obtained using different training samples. 
The cumulative rewards achieved by the six snapshots are given in Table~\ref{table:snapshot}, where snapshot \verb|#|1 is used as the expert policy. Visualization of trajectories obtained by the expert policy in these tasks is provided in Figure~\ref{figure:snapshot_expert}. 
Sourcecode of our datasets and implementation for reproducing the results is publicly available at \textcolor{Blue}{\nolinkurl{https://github.com/voot-t/ril_co}}.

All methods use policy networks with 2 hidden-layers of 64 hyperbolic tangent units. We use similar networks with 100 hyperbolic tangent units for classifiers in RIL-Co and discriminators in other methods. 
The policy networks are trained by ACKTR, where we use a public implementation~\citep{pytorchrl}. 
In each iteration, the policy collects a total of $B=640$ transition samples using $32$ parallel agents, and we use these transition samples as the dataset $\mathcal{B}$ in Algorithm~\ref{algo:rail}. 
Throughout the learning process, the total number of transition samples collected by the learning policy is 20 million.
For training the classifier and discriminator, we use Adam~\citep{KingmaBa2014} with learning rate $10^{-3}$ and the gradient penalty regularizer with the regularization parameter of 10~\citep{GulrajaniAADC17}. 
The mini-batch size for classifier/discriminator training is 128. 

For co-pseudo-labeling in Algorithm~1 of RIL-Co, we initialize by splitting the demonstration dataset $\mathcal{D}$ into two disjoint subset $\mathcal{D}_1$ and $\mathcal{D}_2$.
In each training iteration, we draw batch samples $\mathcal{U} = \{ \vx_u\}_{u=1}^U \sim \mathcal{D}_2$ and $\mathcal{V} = \{ \vx_v\}_{v=1}^V \sim \mathcal{D}_1$ from the split datasets with $U=V=640$.
To obtain pseudo-labeled datasets $\mathcal{P}_1$ for training classifier $g_1$, we firstly compute the classification scores $g_2(\vx_u)$ using samples in $\mathcal{U}$.
Then, we choose $K=128$ samples with the least negative values of $g_2(\vx_u)$ in an ascending order as  $\mathcal{P}_1$.
We choose samples in this way to incorporate a heuristic that prioritizes choosing negative samples which are predicted with high confidence to be negative by the classifiers, i.e., these samples are far away from  the decision boundary.
Without this heuristic, obtaining good approximated samples requires using a large batch size which is computationally expensive. 
The procedure to obtain $\mathcal{P}_2$ is similar, but we use $\mathcal{V}$ instead of $\mathcal{U}$ and $g_1(\vx_v)$ instead of $g_2(\vx_u$).
The implementation of RIL-P variants in our ablation study is similar, except that we have only one neural networks and we do not split the dataset into disjoint subsets. 

For VILD, we the log-sigmoid reward variant and perform important sampling based on the estimated noise, as described by~\cite{TangkarattEtAl2020}. 
For behavior clonig (BC), we use a deterministic policy neural network and train it by minimizing the mean-squared-error with Adam and learning rate $10^{-3}$. We do not apply a regularization technique for BC. 
For the other methods, we follow the original implementation as close as possible, where we make sure that these methods perform well overall on datasets without noise.

\begin{table}[t]
	\centering
	\captionof{table}{Cumulative rewards achieved by six policy snapshots used for generating demonstrations. Snapshot 1 is used as the expert policy. Ant (Gaussian) denotes a scenario in the experiment with the Gaussian noise dataset in Section~\ref{section:exp_gaussian}, where Gaussian noise with different variance is added to expert actions.}
	\label{table:snapshot}
	\begin{tabular}{ l || c | c | c | c | c | c }
		\hline
		Task		  & Snapshot \verb|#|1 & Snapshot \verb|#|2 & Snapshot \verb|#|3 & Snapshot \verb|#|4 & Snapshot \verb|#|5 & Snapshot \verb|#|6 \\ 
		\hline \hline
		\texttt{HalfCheetah}   & 2500	   & 1300		& 1000   	 & 	700		  & -1100	   & -1000		\\
		\texttt{Hopper}   	  & 2300	   & 1100		& 1000		 & 	900		  & 600		   & 0			\\
		\texttt{Walker2D}   	  & 2700	   & 800		& 600		 & 	700		  & 100		   & 0			\\
		\texttt{Ant}   		  & 3500	   & 1400		& 1000		 & 	700		  & 400		   & 0			\\ \hline
		\texttt{Ant} (Gaussian) & 3500	   & 1500		& 1000		 & 	800		  & 500		   & 400		\\
	\end{tabular}
\end{table}

\begin{figure}[t]
	
	\includegraphics[width=0.19\linewidth]{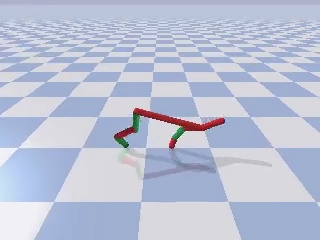} \hfill
	\includegraphics[width=0.19\linewidth]{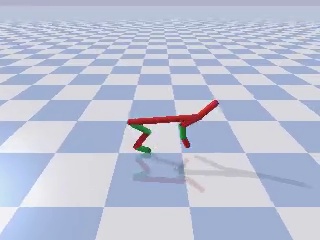} \hfill
	\includegraphics[width=0.19\linewidth]{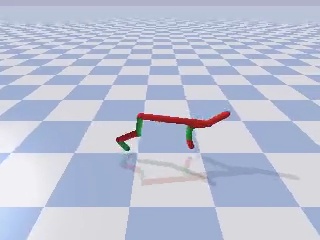} \hfill
	\includegraphics[width=0.19\linewidth]{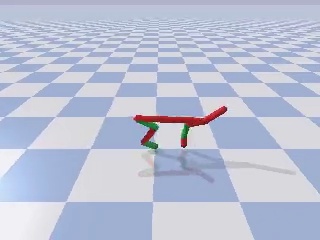} \hfill
	\includegraphics[width=0.19\linewidth]{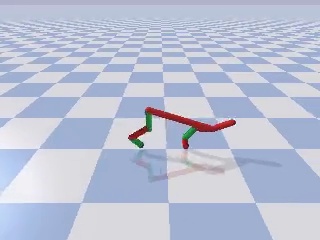} \hfill

	\includegraphics[width=0.19\linewidth]{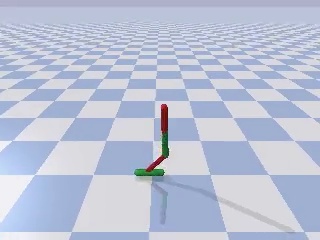} \hfill
	\includegraphics[width=0.19\linewidth]{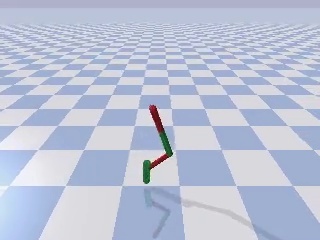} \hfill
	\includegraphics[width=0.19\linewidth]{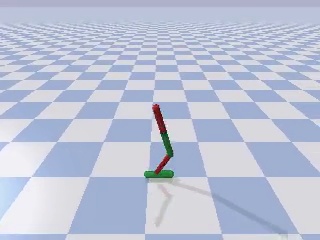} \hfill
	\includegraphics[width=0.19\linewidth]{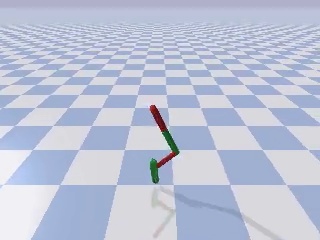} \hfill
	\includegraphics[width=0.19\linewidth]{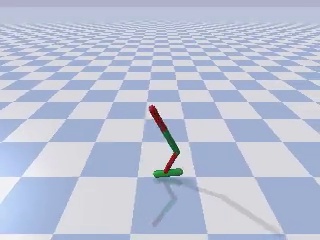} \hfill

	\includegraphics[width=0.19\linewidth]{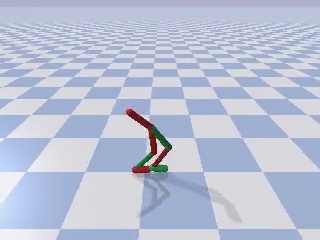} \hfill
	\includegraphics[width=0.19\linewidth]{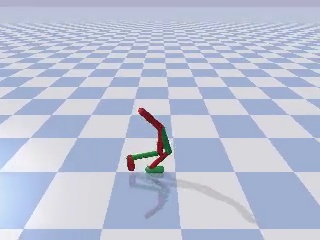} \hfill
	\includegraphics[width=0.19\linewidth]{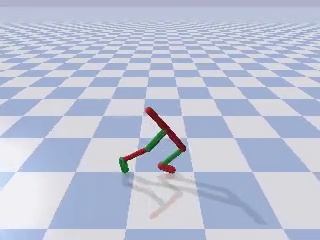} \hfill
	\includegraphics[width=0.19\linewidth]{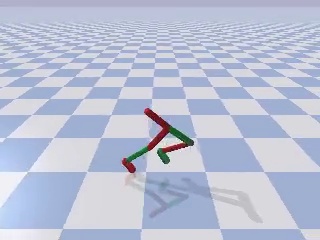} \hfill
	\includegraphics[width=0.19\linewidth]{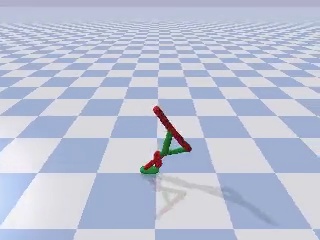} \hfill

	\includegraphics[width=0.19\linewidth]{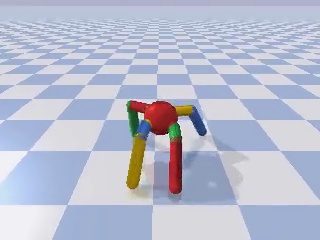} \hfill
	\includegraphics[width=0.19\linewidth]{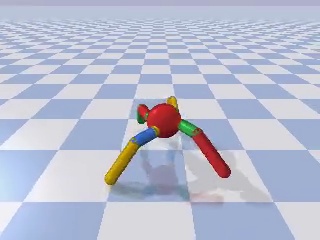} \hfill
	\includegraphics[width=0.19\linewidth]{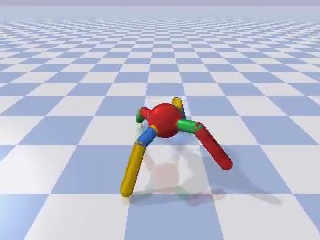} \hfill
	\includegraphics[width=0.19\linewidth]{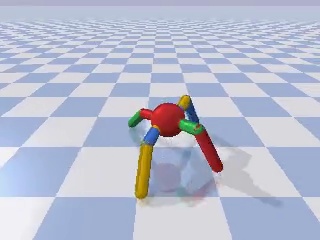} \hfill
	\includegraphics[width=0.19\linewidth]{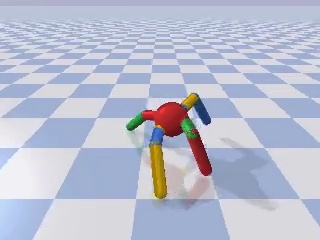} \hfill

	\caption{Visualization of the first 100 time steps of expert demonstrations. Time step increases from the leftmost figure ($t=0)$ to the rightmost figure ($t=100)$. Videos are provided with the sourcecode.
	}	
	\label{figure:snapshot_expert}

\end{figure}

\section{ADDITIONAL RESULTS} 
\label{appendix:add_results}

Here, we present learning curves obtained by each method in the experiments. 
Figure~\ref{figure:exp_app_curve} depicts learning curves (performance against the number of transition samples collected by the learning policy) for the results in Section~\ref{section:exp_main}.
Since BC does not use the learning policy to collect transition samples, the horizontal axes for BC denote the number of training iterations. 
It can be seen that RIL-Co achieves better performances and uses less transition samples compared to other methods in the high-noise scenarios where  $\delta \in \{ 0.2, 0.3, 0.4\}$.
Meanwhile in the low-noise scenarios, all methods except GAIL with the unhinged loss and VILD perform comparable to each other in terms of the performance and sample efficiency. 

Figure~\ref{figure:exp_app_ablation_curve} shows learning curves of the ablation study in Section~\ref{section:exp_ablation}. 
RIL-Co with the AP loss clearly outperforms the comparison methods in terms of both sample efficiency and final performance.
The final performance in Figures~\ref{figure:exp_main} and~\ref{figure:exp_ablation_noise} are obtained by averaging the performance in the last 1000 iterations of the learning curves.  

\begin{figure}[t]	
	\centering
	\includegraphics[width=0.92\linewidth]{figures/legend.pdf}

	\vspace{1mm}

	\begin{subfigure}[b]{0.99\linewidth}
		\includegraphics[width=0.24\linewidth]{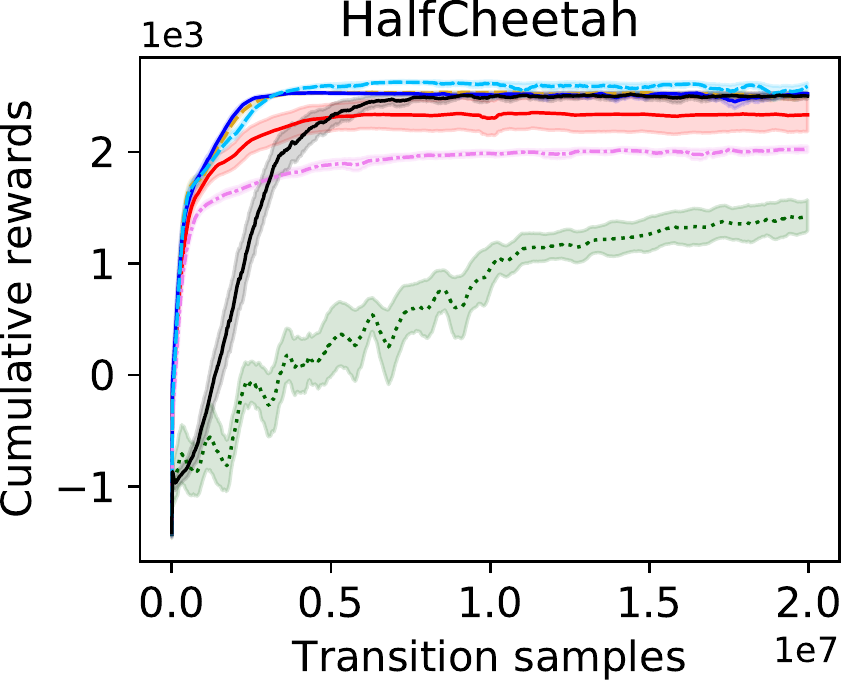} \hfill
		\includegraphics[width=0.24\linewidth]{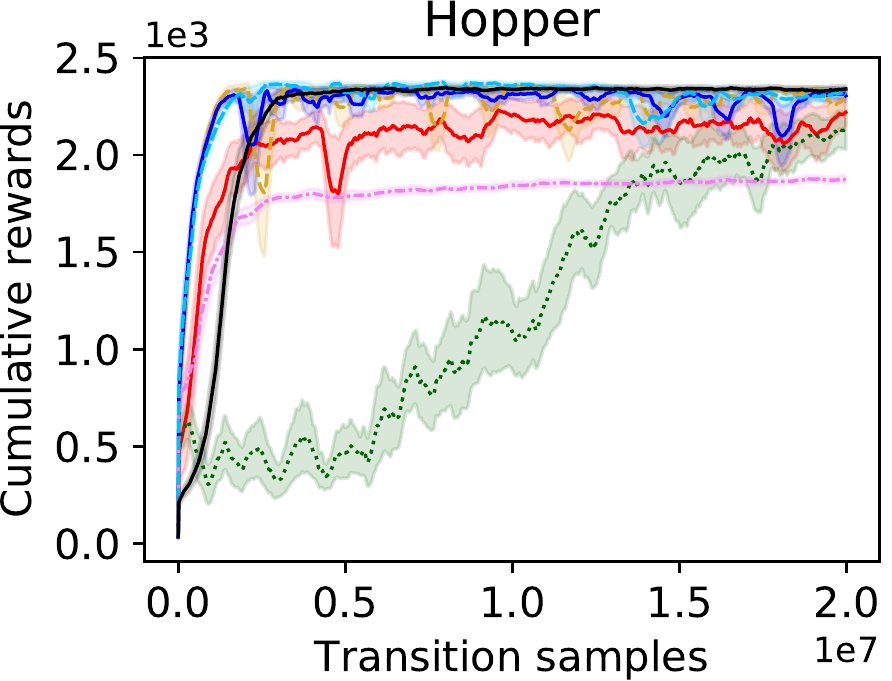} \hfill
		\includegraphics[width=0.24\linewidth]{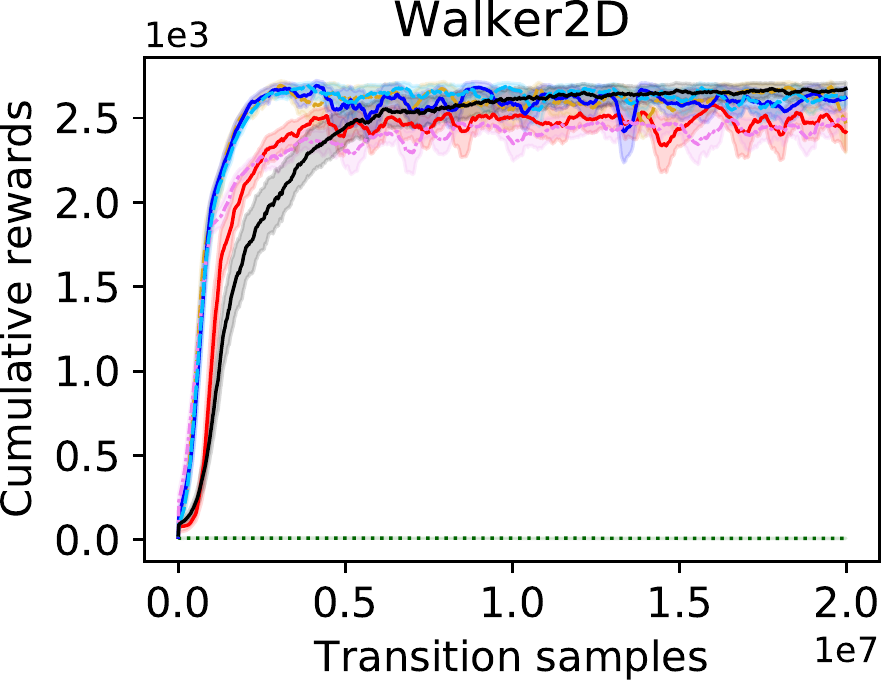} \hfill
		\includegraphics[width=0.24\linewidth]{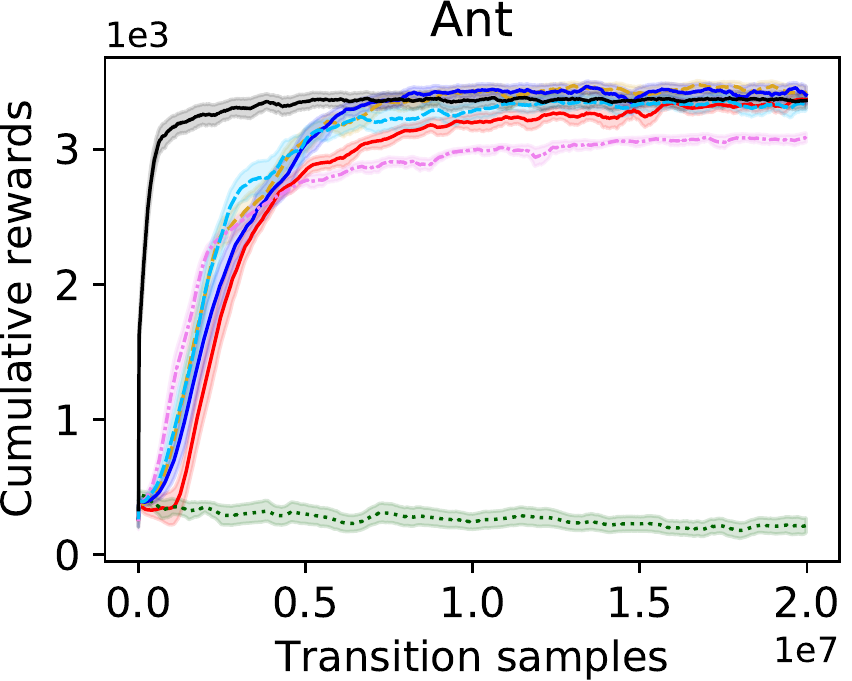} \hfill
		\subcaption{Noise rate $\delta=0$}
	\end{subfigure}
	
	\begin{subfigure}[b]{0.99\linewidth}
		\includegraphics[width=0.24\linewidth]{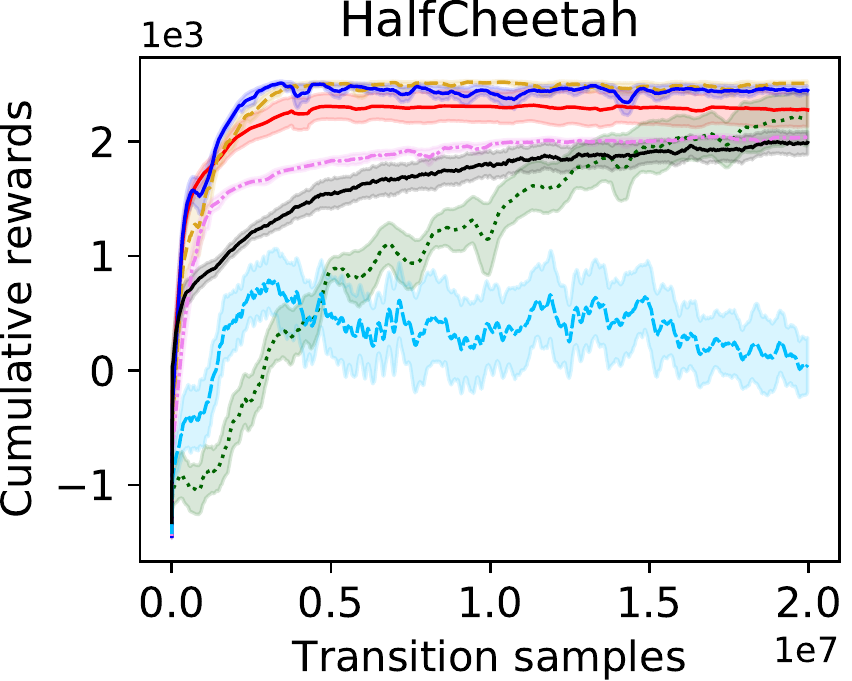} \hfill
		\includegraphics[width=0.24\linewidth]{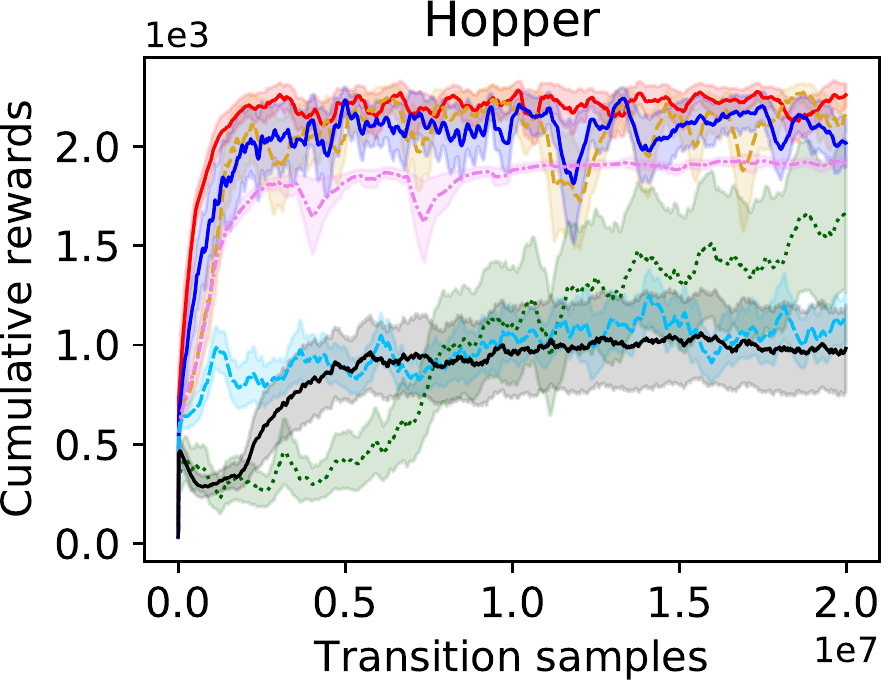} \hfill
		\includegraphics[width=0.24\linewidth]{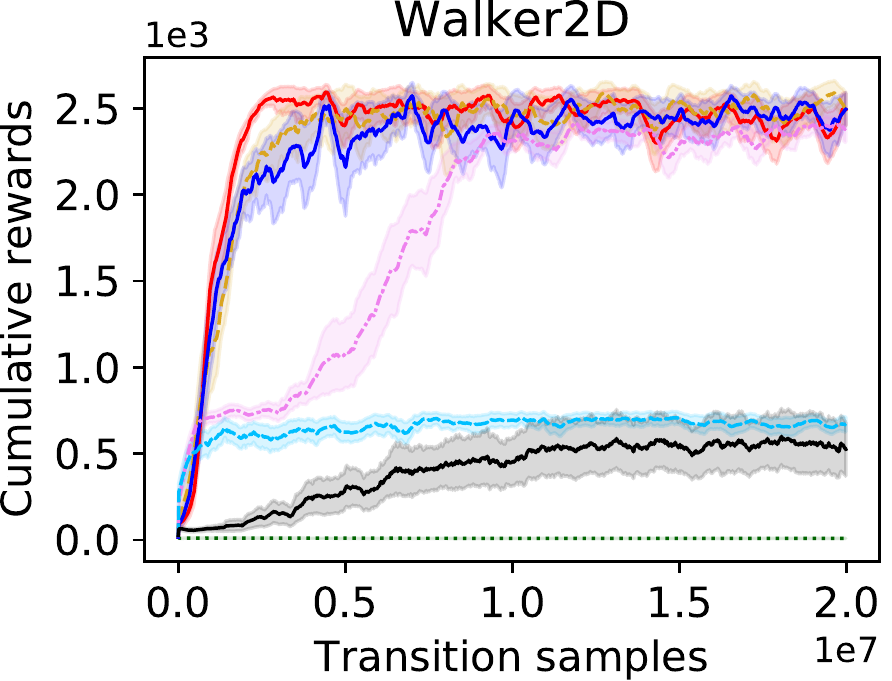} \hfill
		\includegraphics[width=0.24\linewidth]{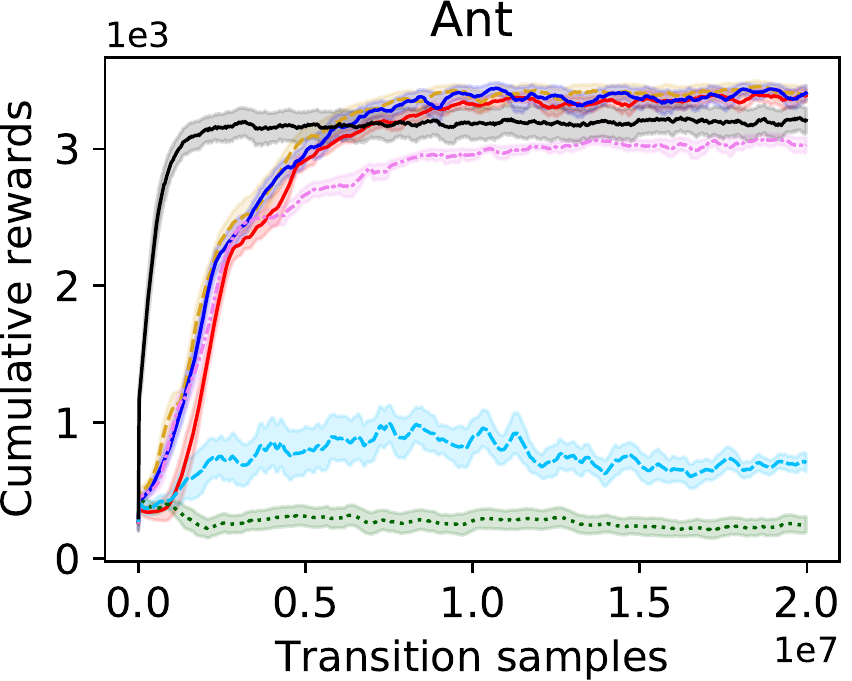} \hfill
		\subcaption{Noise rate $\delta=0.1$}
	\end{subfigure}
	
	\begin{subfigure}[b]{0.99\linewidth}
		\includegraphics[width=0.24\linewidth]{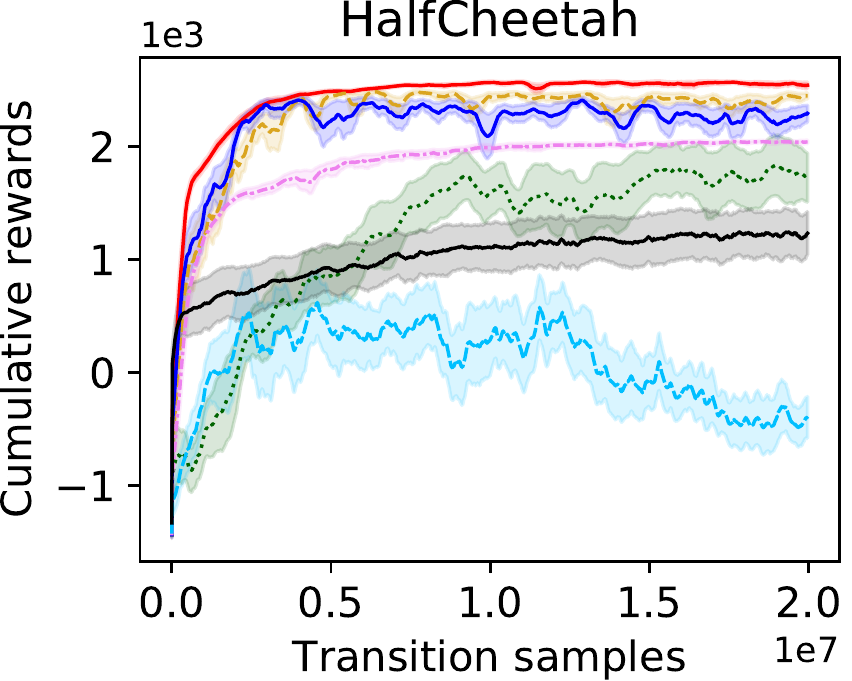} \hfill
		\includegraphics[width=0.24\linewidth]{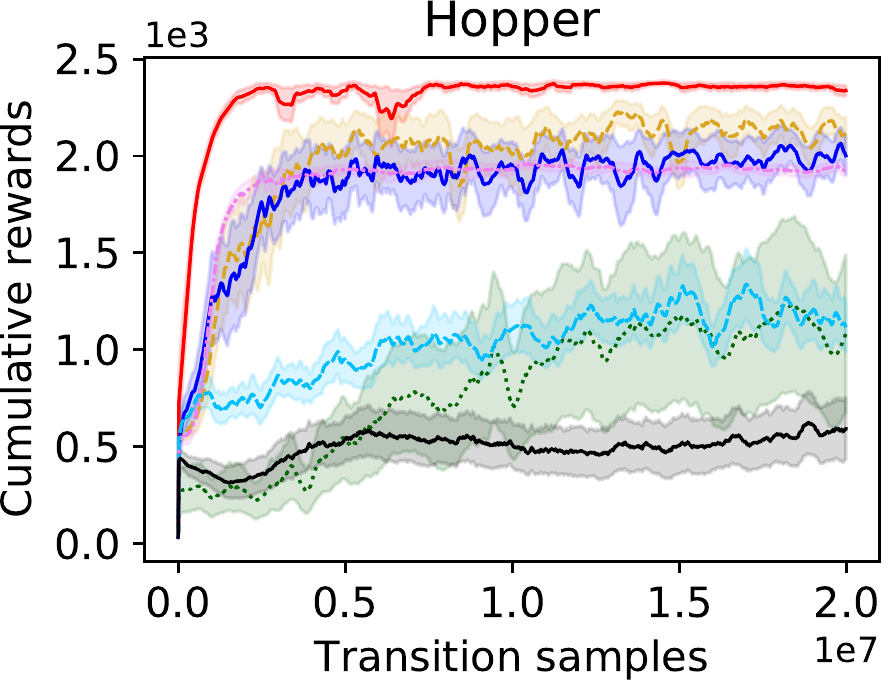} \hfill
		\includegraphics[width=0.24\linewidth]{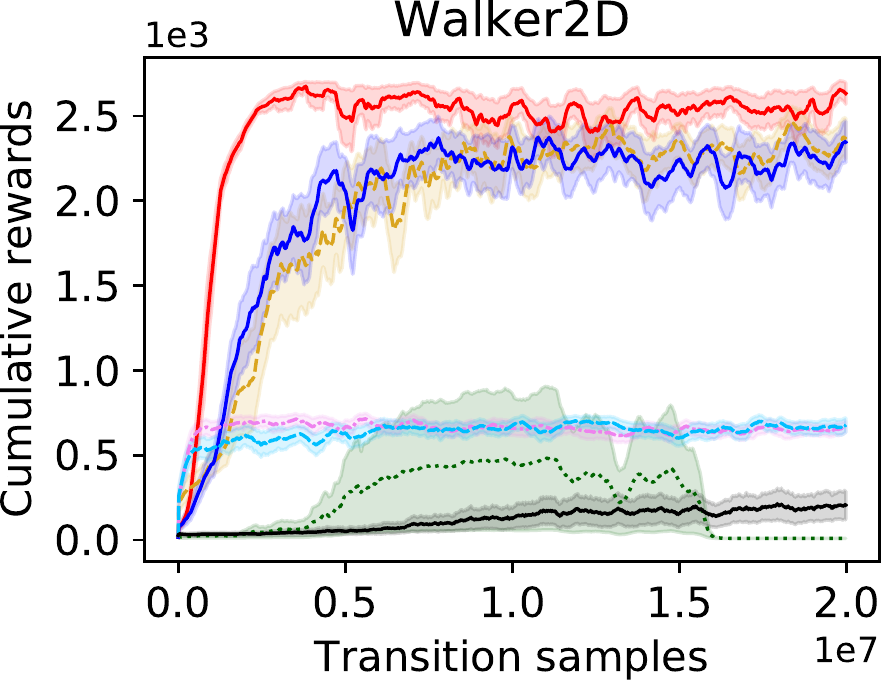} \hfill
		\includegraphics[width=0.24\linewidth]{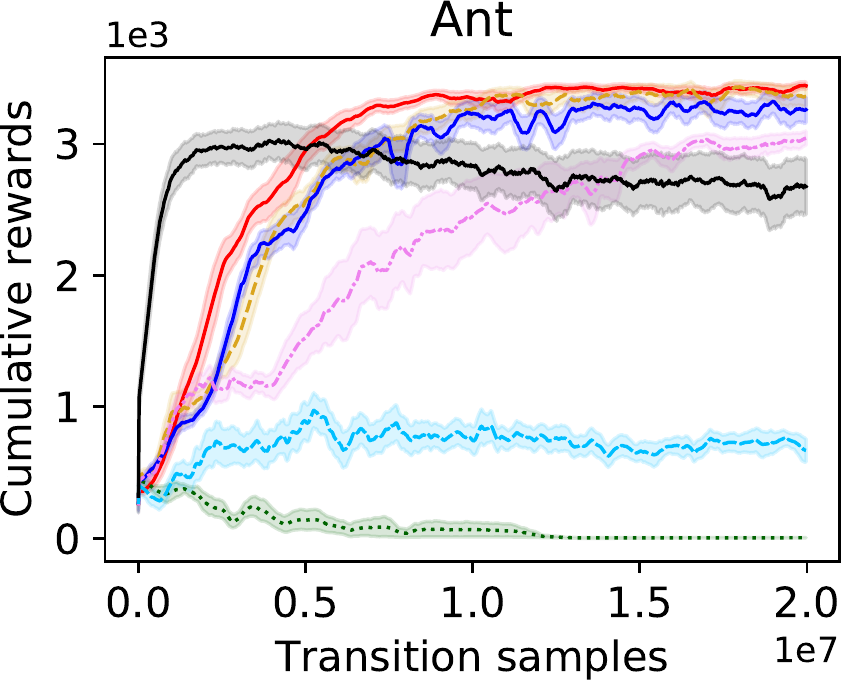} \hfill
		\subcaption{Noise rate $\delta=0.2$}
	\end{subfigure}
	
	\begin{subfigure}[b]{0.99\linewidth}
		\includegraphics[width=0.24\linewidth]{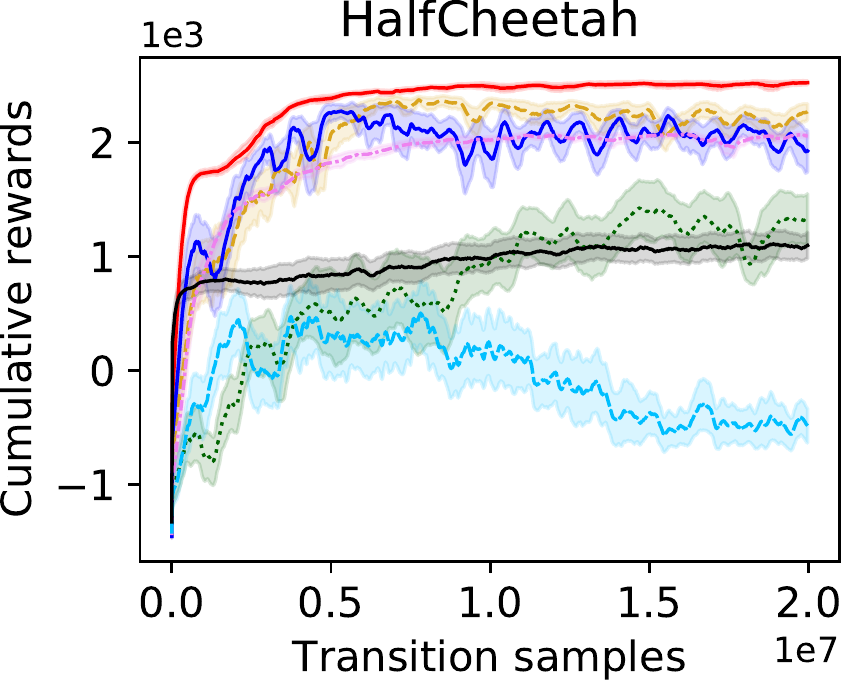} \hfill
		\includegraphics[width=0.24\linewidth]{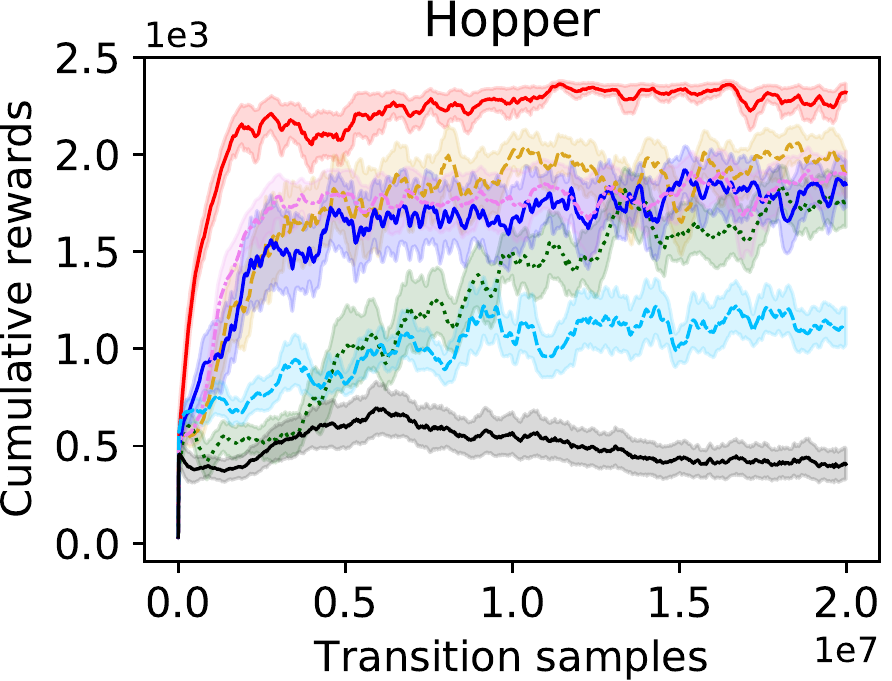} \hfill
		\includegraphics[width=0.24\linewidth]{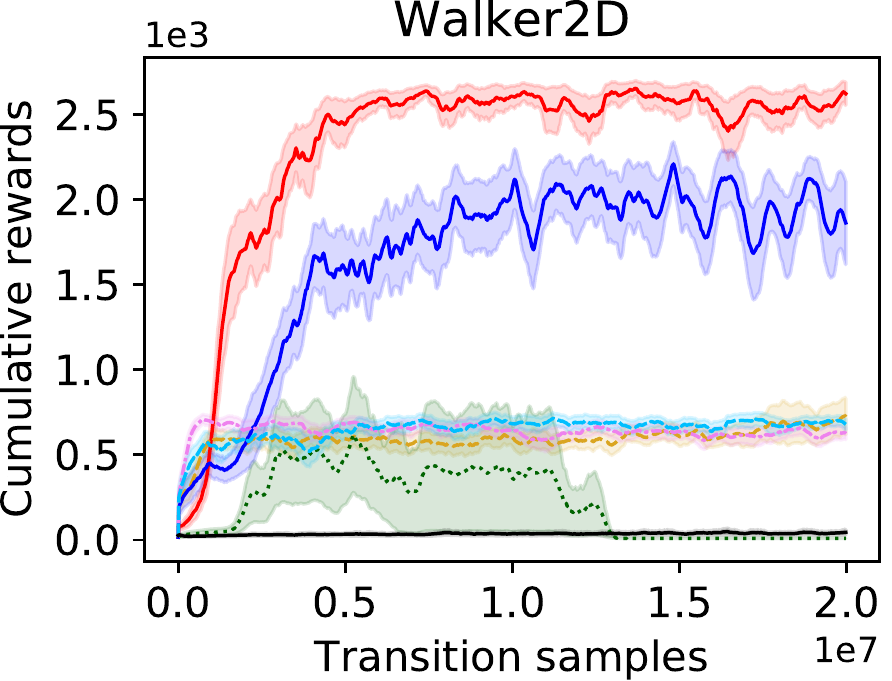} \hfill
		\includegraphics[width=0.24\linewidth]{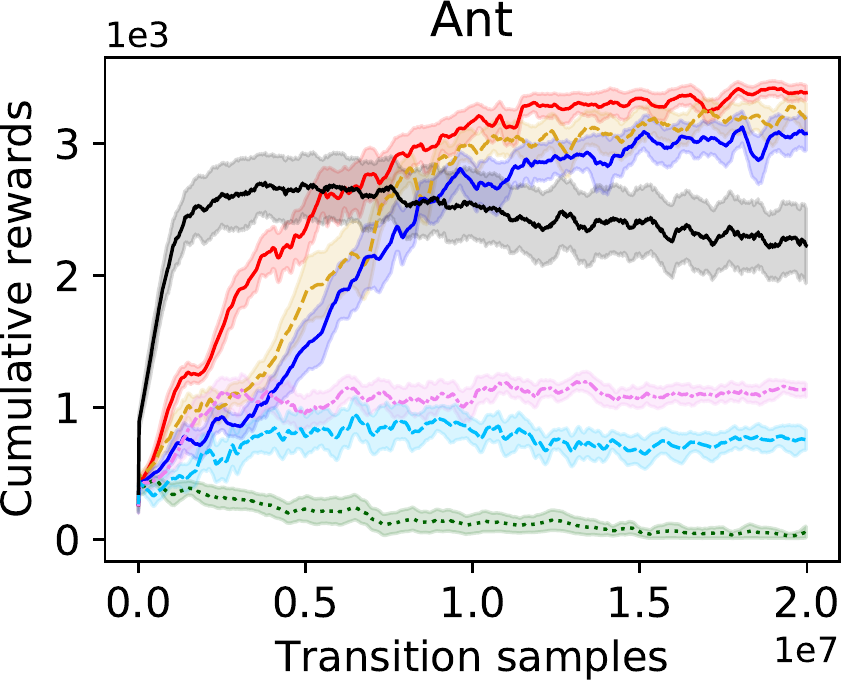} \hfill
		\subcaption{Noise rate $\delta=0.3$}
	\end{subfigure}
	
	\begin{subfigure}[b]{0.99\linewidth}
	\includegraphics[width=0.24\linewidth]{figures/ACKTR_HalfCheetahBulletEnv-v0_np0.4.pdf} \hfill
	\includegraphics[width=0.24\linewidth]{figures/ACKTR_HopperBulletEnv-v0_np0.4.pdf} \hfill
	\includegraphics[width=0.24\linewidth]{figures/ACKTR_Walker2DBulletEnv-v0_np0.4.pdf} \hfill
	\includegraphics[width=0.24\linewidth]{figures/ACKTR_AntBulletEnv-v0_np0.4.pdf} \hfill
	\subcaption{Noise rate $\delta=0.4$}
	\end{subfigure}

	\caption{Performance against the number of transition samples in continuous-control benchmarks. For BC, the horizontal axes denote the number of training iterations. Clearly, RIL-Co with the AP loss is more robust than comparison methods when the noise rate increases.}	
	\label{figure:exp_app_curve}
\end{figure}

\begin{figure}[t]	
	\centering
	\includegraphics[width=0.70\linewidth]{figures/legend_app.pdf}

	\vspace{1mm}

	\begin{subfigure}[b]{0.99\linewidth}
		\includegraphics[width=0.24\linewidth]{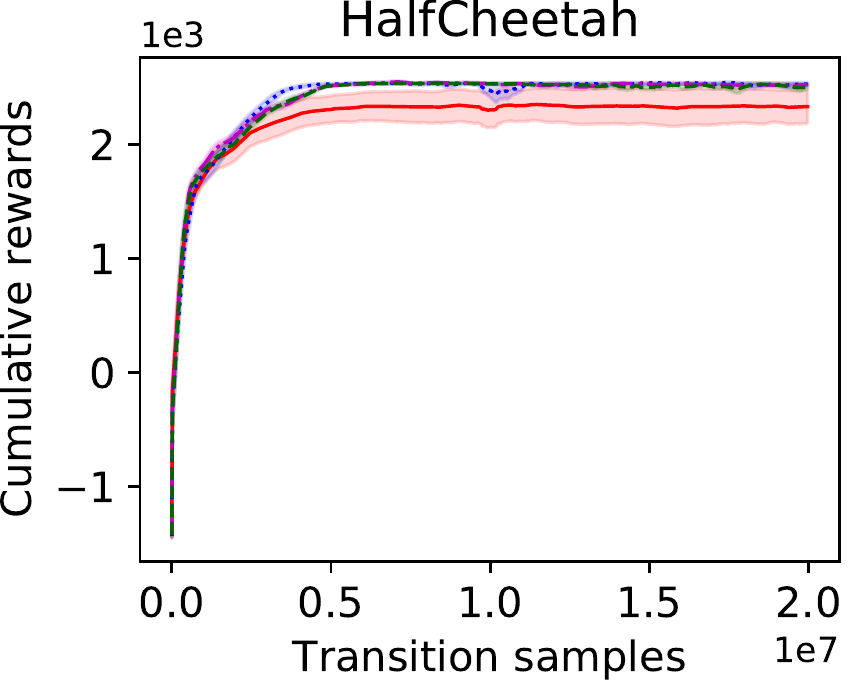} \hfill
		\includegraphics[width=0.24\linewidth]{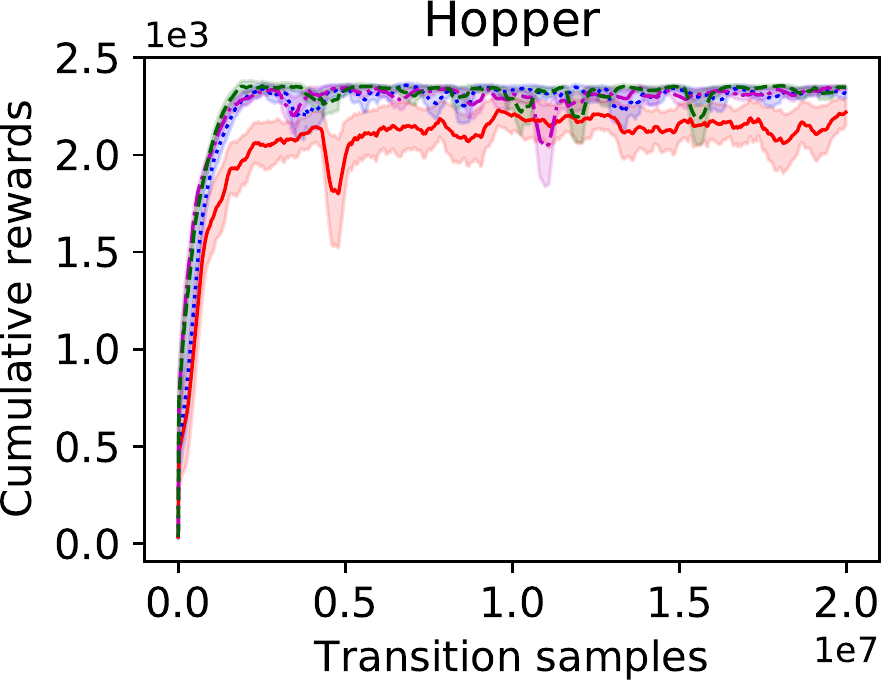} \hfill
		\includegraphics[width=0.24\linewidth]{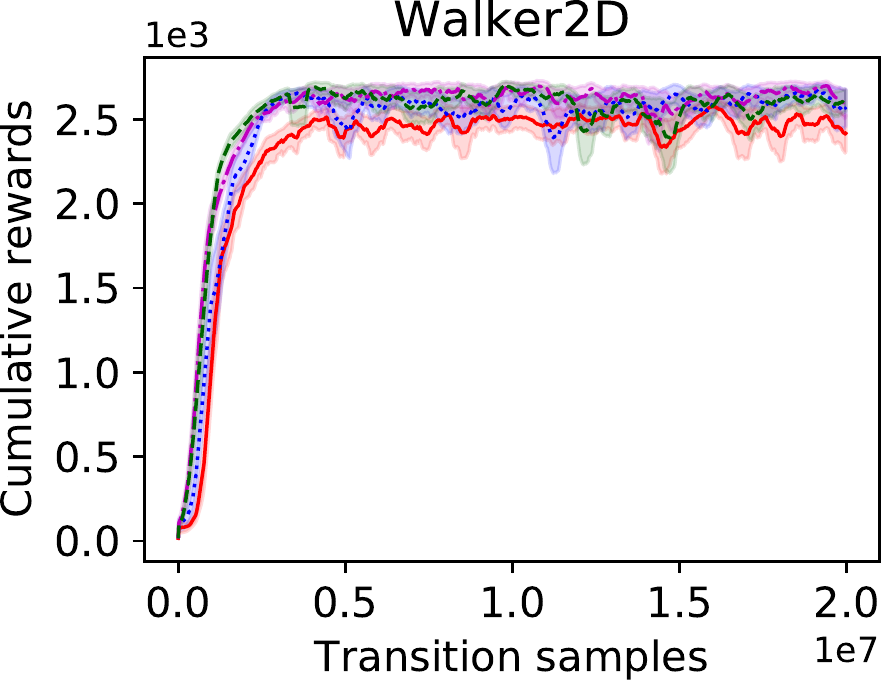} \hfill
		\includegraphics[width=0.24\linewidth]{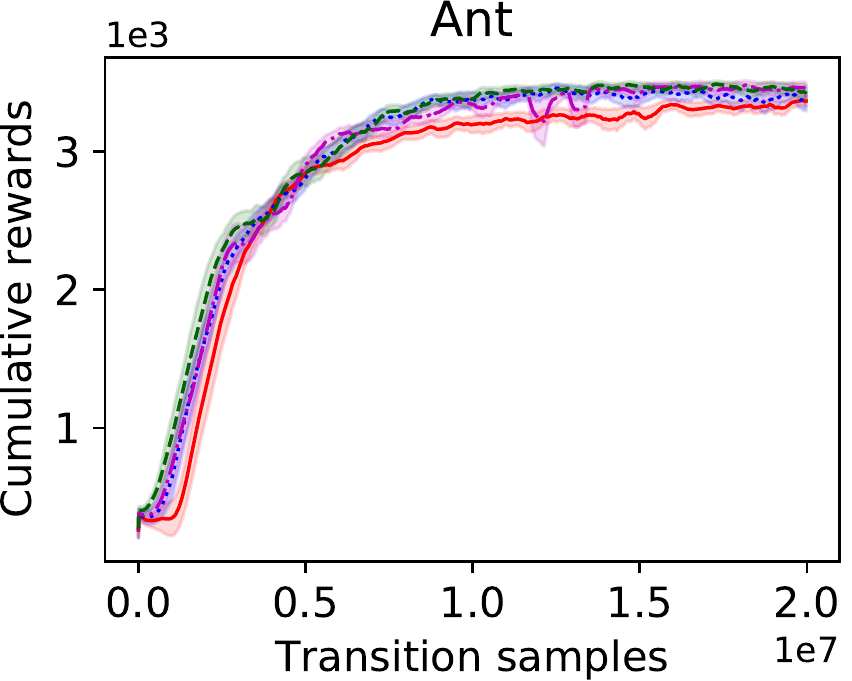} \hfill
		\subcaption{Noise rate $\delta=0$}
	\end{subfigure}
	
	\begin{subfigure}[b]{0.99\linewidth}
		\includegraphics[width=0.24\linewidth]{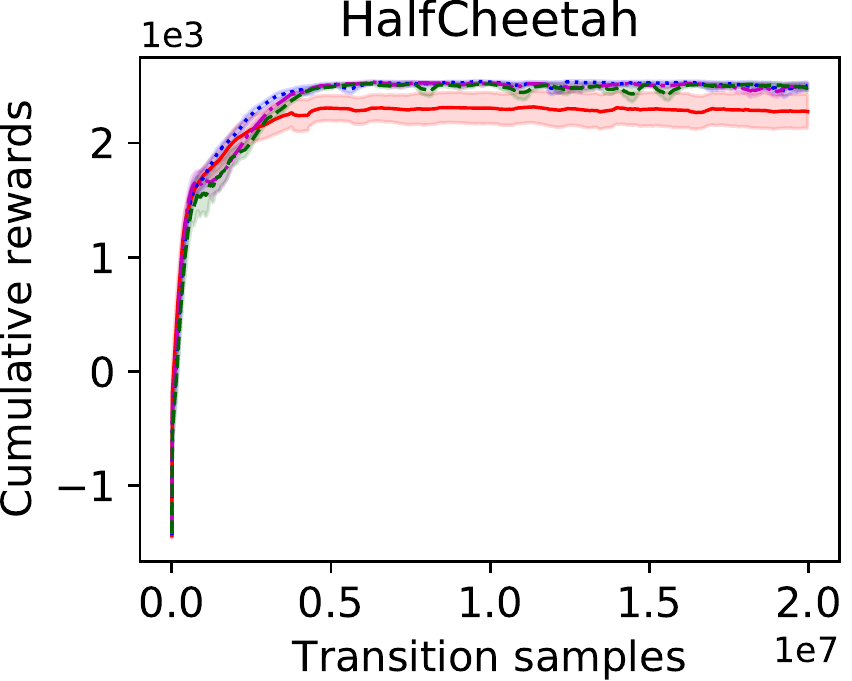} \hfill
		\includegraphics[width=0.24\linewidth]{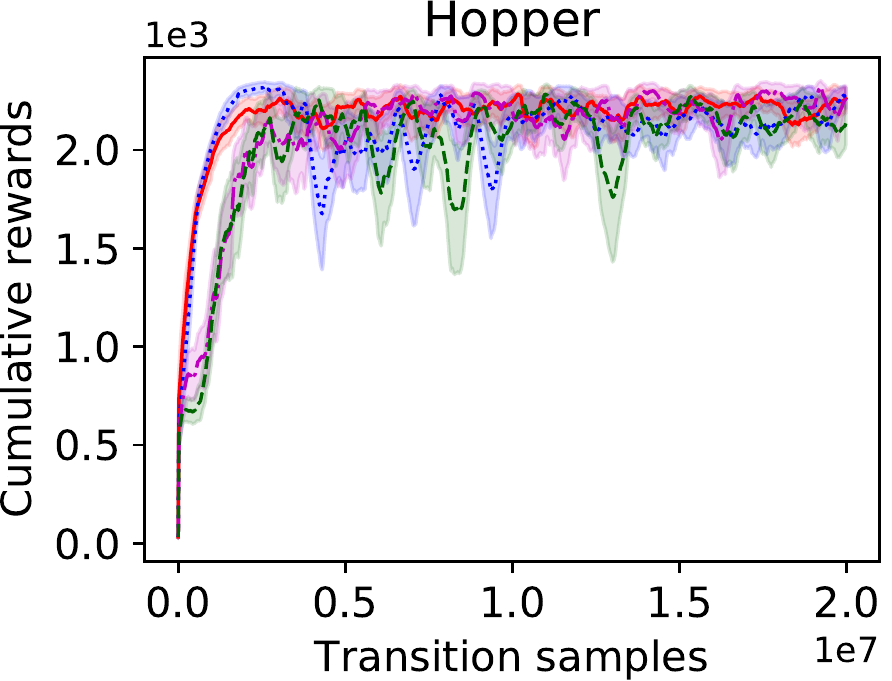} \hfill
		\includegraphics[width=0.24\linewidth]{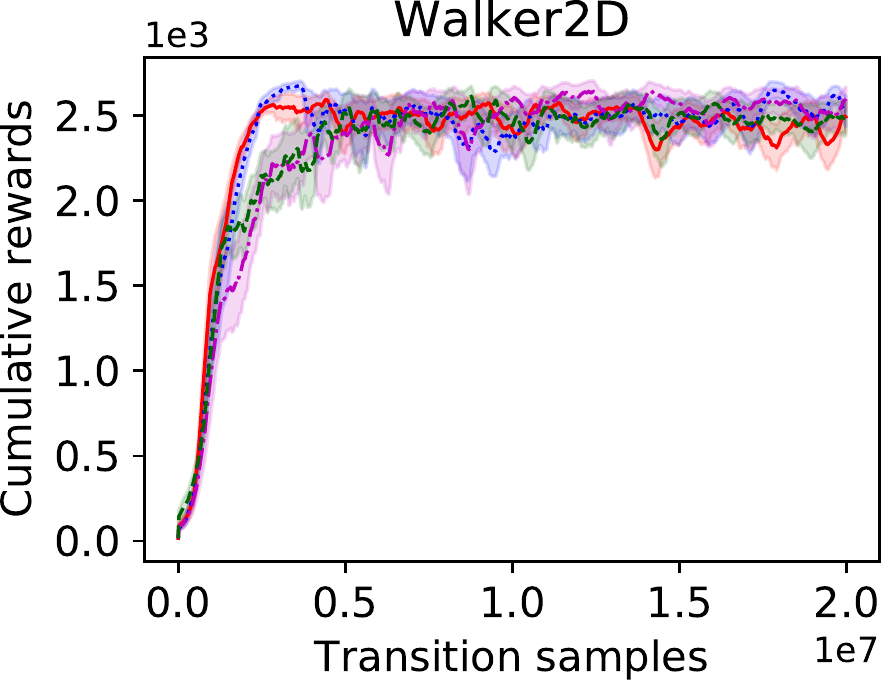} \hfill
		\includegraphics[width=0.24\linewidth]{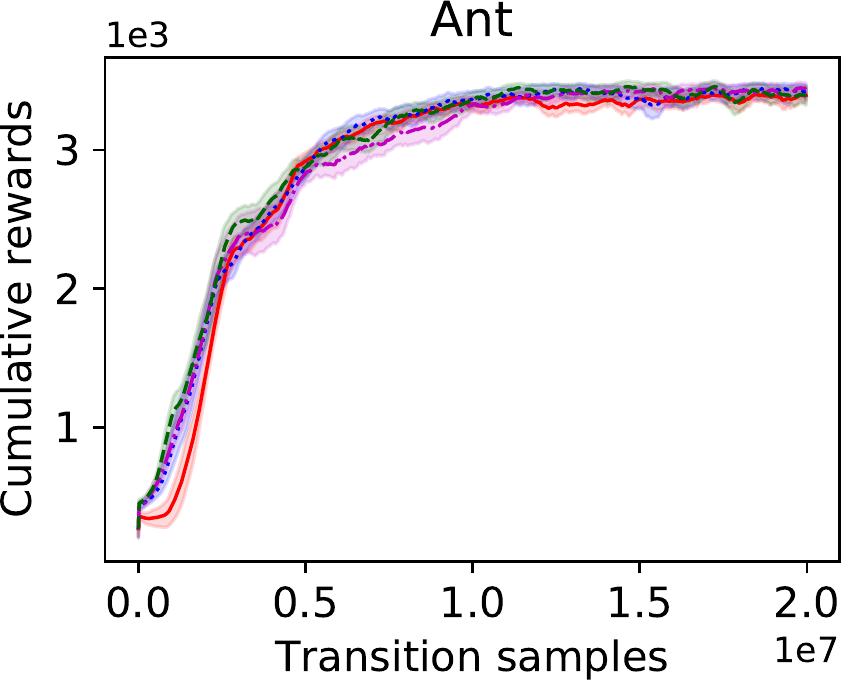} \hfill
		\subcaption{Noise rate $\delta=0.1$}
	\end{subfigure}
	
	\begin{subfigure}[b]{0.99\linewidth}
		\includegraphics[width=0.24\linewidth]{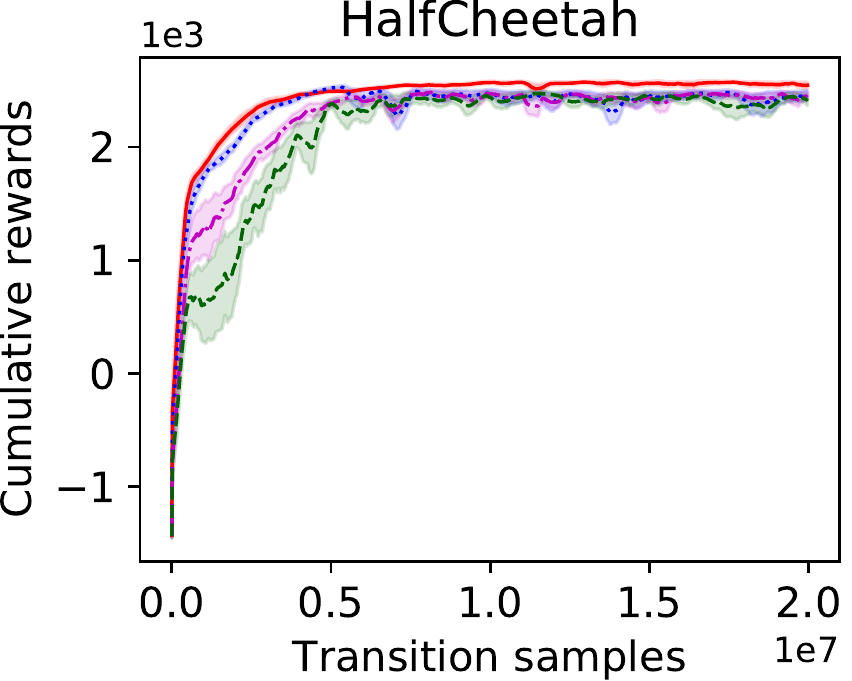} \hfill
		\includegraphics[width=0.24\linewidth]{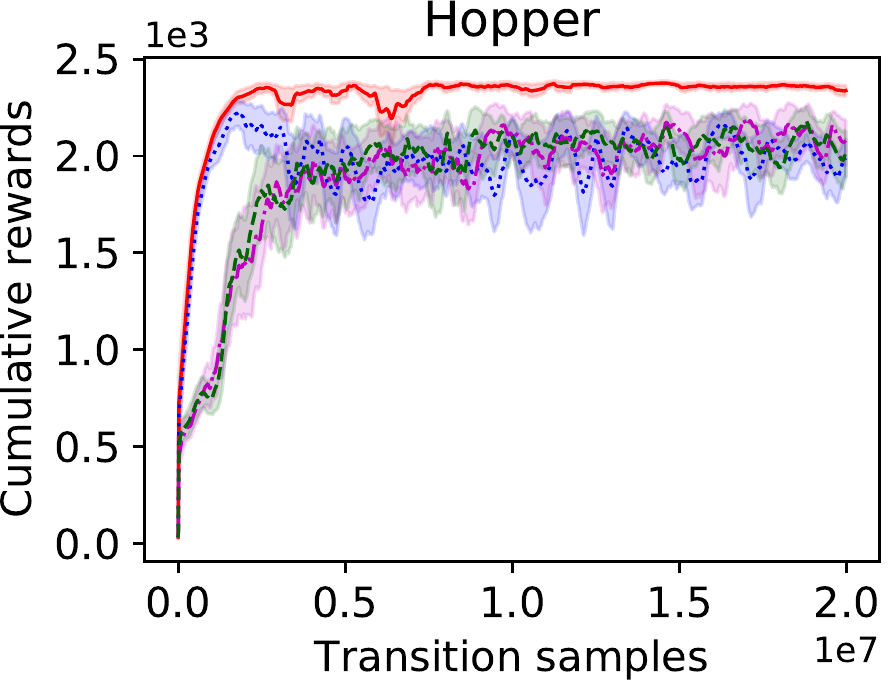} \hfill
		\includegraphics[width=0.24\linewidth]{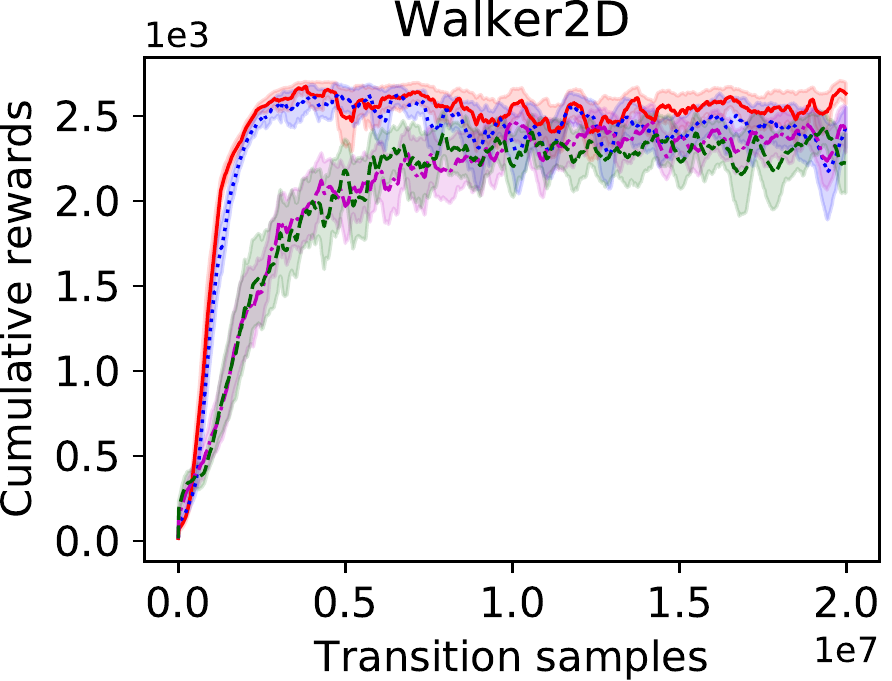} \hfill
		\includegraphics[width=0.24\linewidth]{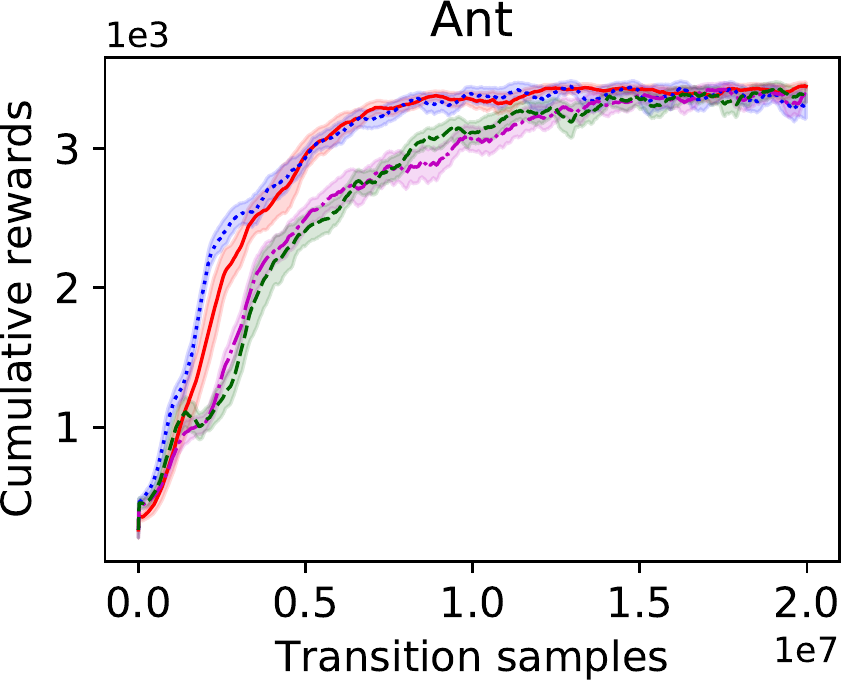} \hfill
		\subcaption{Noise rate $\delta=0.2$}
	\end{subfigure}
	
	\begin{subfigure}[b]{0.99\linewidth}
		\includegraphics[width=0.24\linewidth]{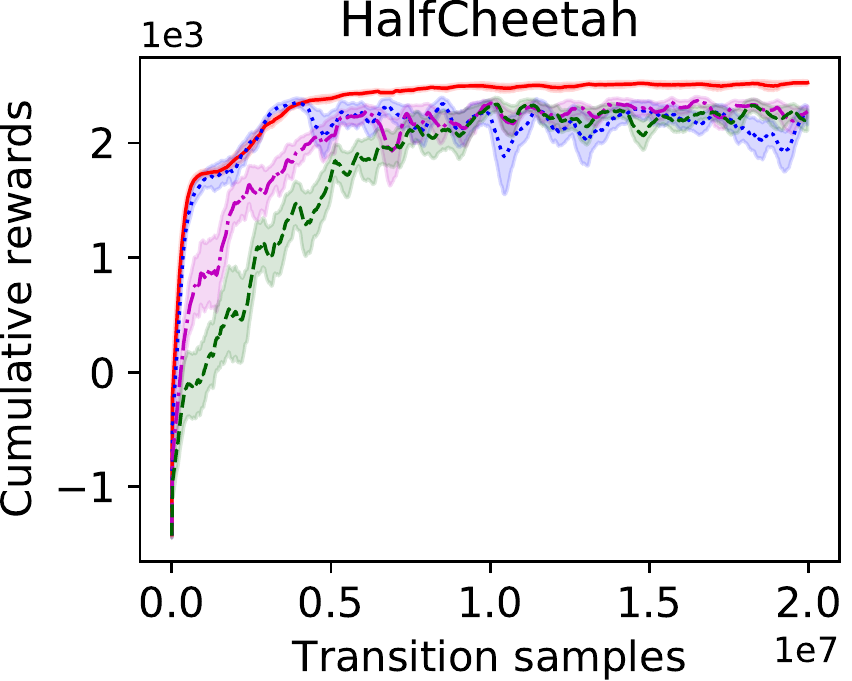} \hfill
		\includegraphics[width=0.24\linewidth]{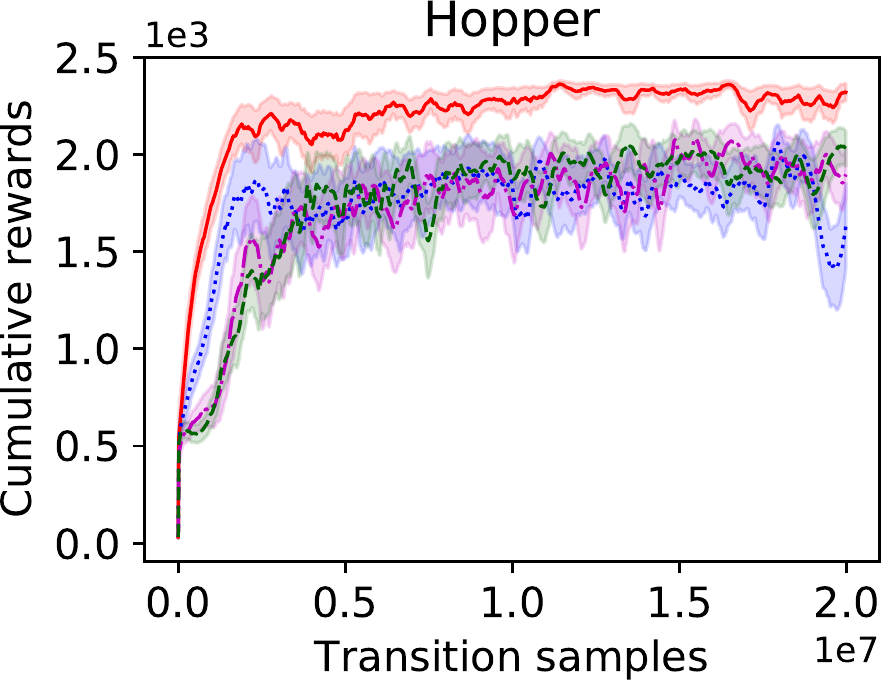} \hfill
		\includegraphics[width=0.24\linewidth]{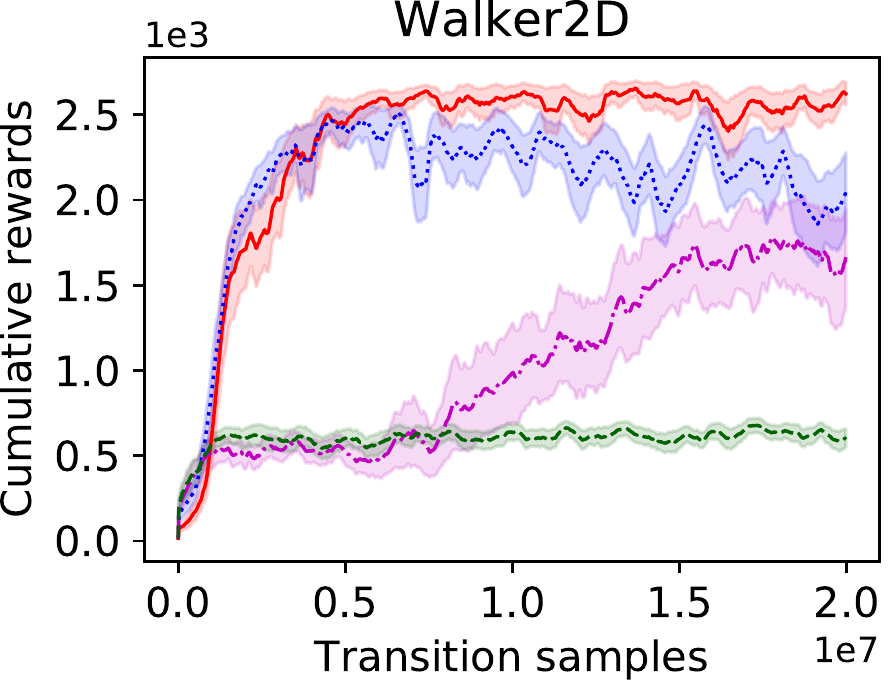} \hfill
		\includegraphics[width=0.24\linewidth]{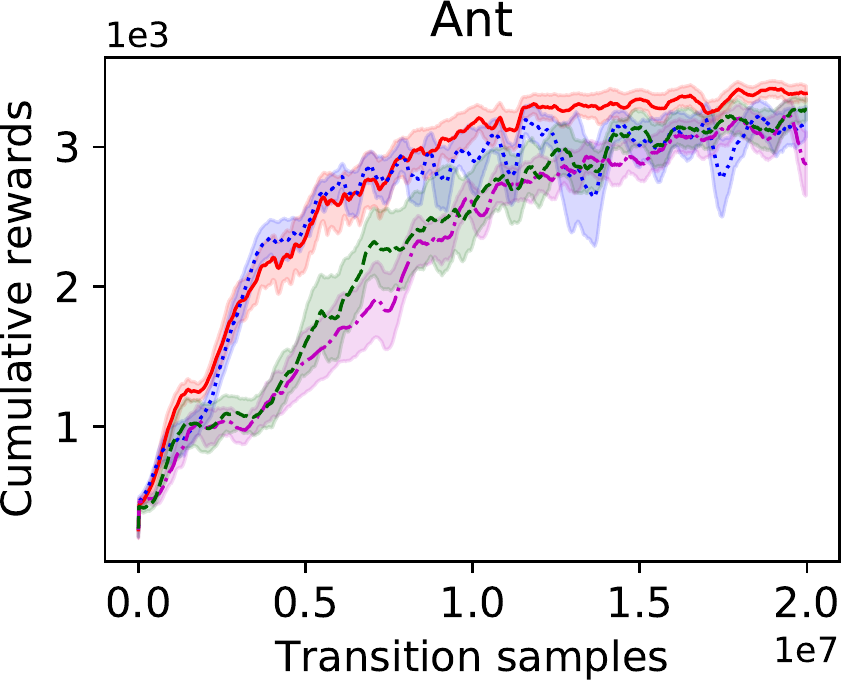} \hfill
		\subcaption{Noise rate $\delta=0.3$}
	\end{subfigure}
	
	\begin{subfigure}[b]{0.99\linewidth}
	\includegraphics[width=0.24\linewidth]{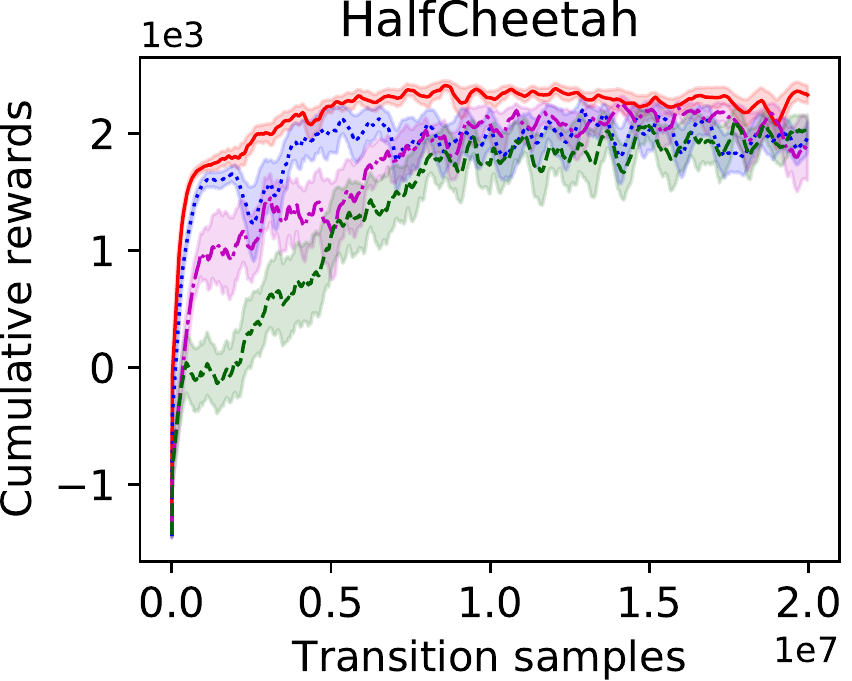} \hfill
	\includegraphics[width=0.24\linewidth]{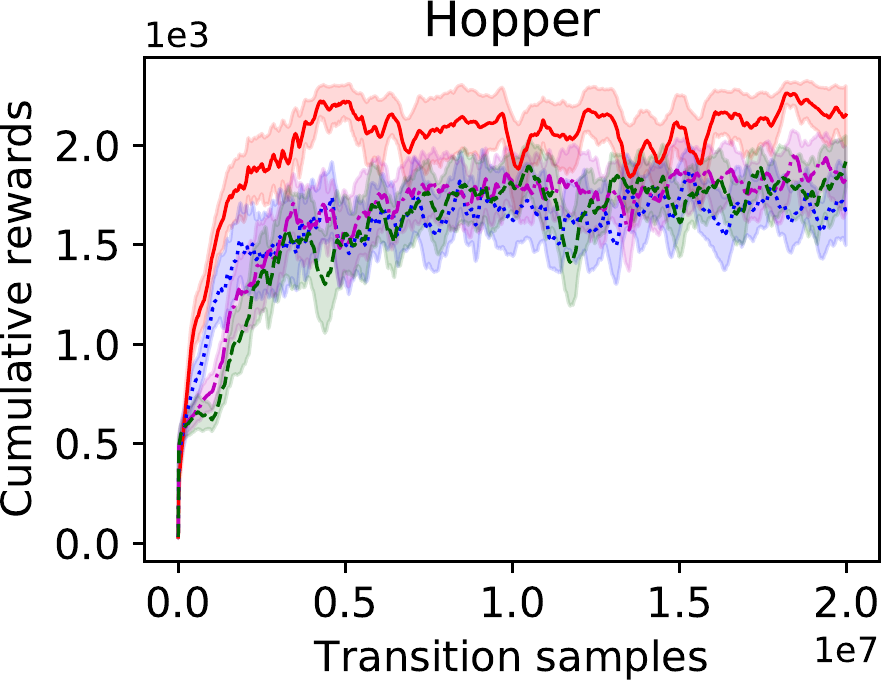} \hfill
	\includegraphics[width=0.24\linewidth]{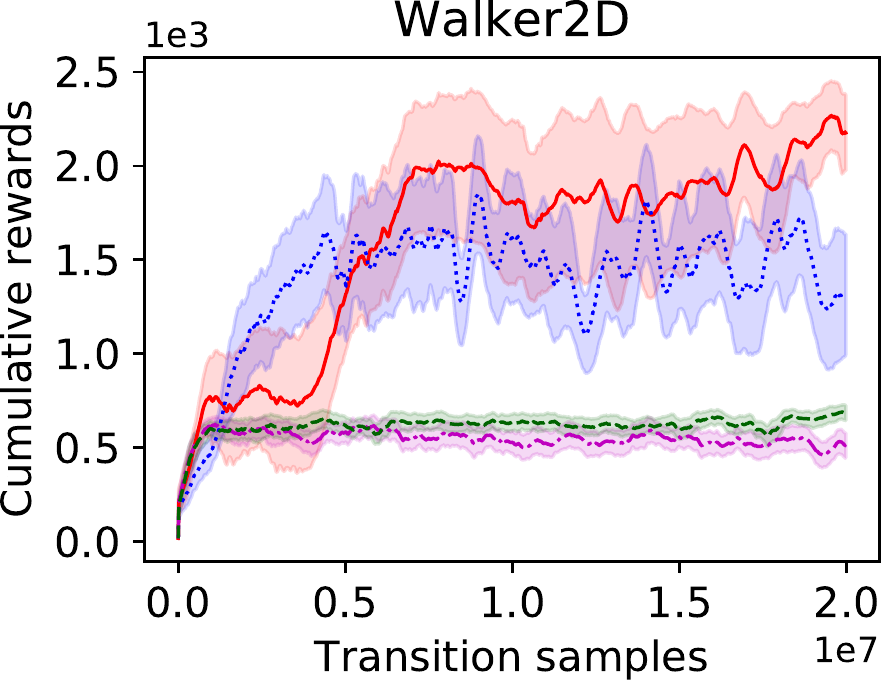} \hfill
	\includegraphics[width=0.24\linewidth]{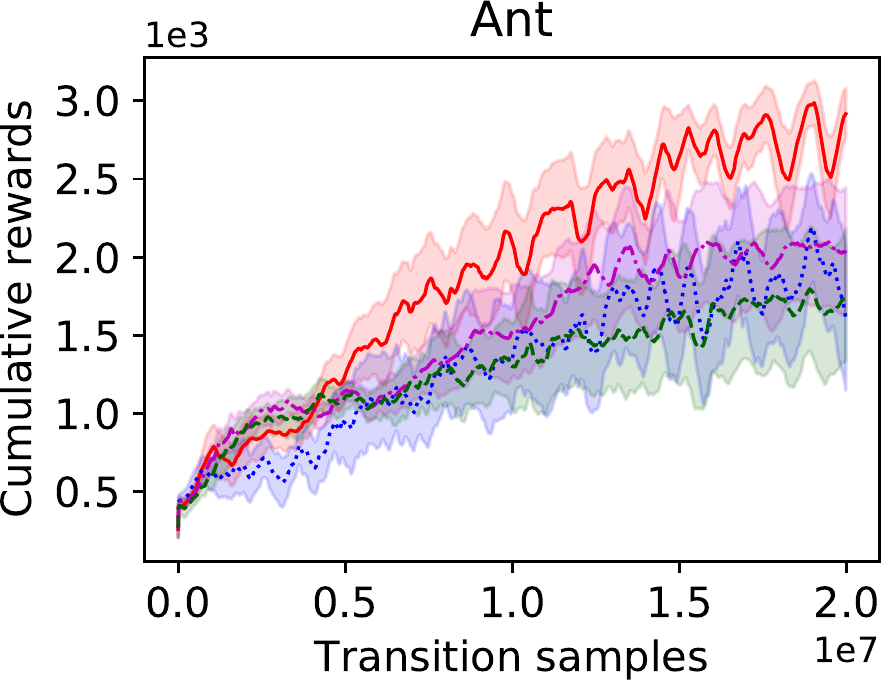} \hfill
	\subcaption{Noise rate $\delta=0.4$}
	\end{subfigure}

	\caption{Performance against the number of transition samples in the ablation study. RIL-Co with the AP loss is more robust than its variants that use non-symmetric losses and naive pseudo-labeling.}	
	\label{figure:exp_app_ablation_curve}
\end{figure}

\end{document}